\documentclass[10pt,twocolumn,letterpaper]{article}

 \usepackage[pagenumbers]{cvpr} % To force page numbers, e.g. for an arXiv version

\definecolor{blue1}{RGB}{0, 102, 204}
\definecolor{tealblue}{RGB}{0, 132, 194}
\definecolor{darkgreen}{rgb}{0.0, 0.5, 0.0}
\definecolor{codegray}{rgb}{0.5,0.5,0.5}
\definecolor{codekw}{rgb}{0.85, 0.18, 0.50}
\usepackage{listings}

\usepackage{booktabs}       % professional-quality tables
\usepackage{amsfonts}       % blackboard math symbols
\usepackage{nicefrac}       % compact symbols for 1/2, etc.
\usepackage{microtype}      % microtypography
\usepackage{xcolor}         % colors
\usepackage{enumitem}
\usepackage{microtype}
\usepackage{graphicx}
\usepackage{subfloat}
\usepackage{multirow}
\usepackage{multicol}
\usepackage{algorithm}
\usepackage{algorithmic}
\usepackage{makecell}
\usepackage{colortbl}

\newcommand{\best}[1]{\textcolor{black}{\textbf{#1}}}
\newcommand{\second}[1]{\textcolor{black}{\underline{#1}}}
\definecolor{cvprblue}{rgb}{0.21,0.49,0.74}
\usepackage[pagebackref,breaklinks,colorlinks,allcolors=cvprblue]{hyperref}
\usepackage{amsmath,amsfonts,bm}

\DeclareMathAlphabet{\mathsfit}{\encodingdefault}{\sfdefault}{m}{sl}
\SetMathAlphabet{\mathsfit}{bold}{\encodingdefault}{\sfdefault}{bx}{n}

\newcommand{\R}{\mathbb{R}}

\newcommand{\normlp}{l_p}
\newcommand{\normmax}{l_\infty}

\newcommand{\loss}{\ell}

\DeclareMathOperator*{\probability}{\text{Pr.}}
\DeclareMathOperator*{\naturalrisk}{\mathcal{R}_{\text{nat}}}
\DeclareMathOperator*{\robustrisk}{\mathcal{R}_{\text{rob}}}
\DeclareMathOperator*{\balrobustrisk}{\bar{\mathcal{R}}_{\text{rob}}}

\newcommand{\naturalzscore}{\mathcal{Z}_{\text{nat}}}
\newcommand{\robustzscore}{\mathcal{Z}_{\text{rob}}}

\DeclareMathOperator*{\argmax}{arg\,max}

\DeclareMathOperator{\sign}{sign}

\usepackage{amsmath}
\usepackage{amssymb}
\usepackage{mathtools}
\usepackage{amsthm}
\usepackage{bm}

\usepackage{dsfont}
\newcommand{\indicator}[1]{\mathds{1}{#1}}

\usepackage[capitalize,noabbrev]{cleveref}

\theoremstyle{plain}
\newtheorem{theorem}{Theorem}[section]

\newtheorem{lemma}[theorem]{Lemma}

\newtheorem{definition}{Definition}[section]
\theoremstyle{remark}
\newtheorem{remark}{Remark}[section]

\crefname{theorem}{Theorem}{Theorems}
\crefname{lemma}{Lemma}{Lemmas}
\crefname{proposition}{Proposition}{Propositions}
\crefname{corollary}{Corollary}{Corollaries}
\crefname{definition}{Definition}{Definitions}
\crefname{assumption}{Assumption}{Assumptions}

\AtEndPreamble{%
  \crefname{table}{Table}{Tables}%
  \Crefname{table}{Table}{Tables}%
}

\def\paperID{-} % *** Enter the Paper ID here
\def\confName{CVPR}
\def\confYear{2026}

\def\TITLE{Taming the Long Tail: Rebalancing Adversarial Training via\\ Adaptive Perturbation}
\title{\TITLE}
\def\METHODNAME{RobustLT}

\author{
Lilin Zhang$^{1}$
\quad
Yimo Guo$^{1}$
\quad
Yue Li$^{2}$
\quad
Jiancheng Shi$^{3}$
\quad
Xianggen Liu$^{1}$\thanks{Corresponding author}\\
$^{1}$Sichuan University
\quad
$^{2}$Dongfang Electric (Chengdu) Innovation Research Co., Ltd. \\
$^{3}$Southwest China Research Institute of Electronic Equipment \\
\makecell[c]{\tt\small zhanglilin@stu.scu.edu.cn \quad guoyimo@stu.scu.edu.cn \quad liy1383@dongfang.com \\ \tt\small shijcbit@163.com \quad liuxianggen@scu.edu.cn}
}

\begin{document}
\maketitle
\begin{abstract}
Deep neural networks are highly vulnerable to adversarial examples, i.e.,small perturbations that can significantly degrade model performance. While adversarial training has become the primary defense strategy, most studies focus on balanced datasets, overlooking the challenges posed by real-world long-tail data. Motivated by the fact that perturbations in adversarial examples inherently alter the training distribution, we theoretically investigate their impact. We first revisit adversarial training for long-tail data and identify two key limitations: (i) a skewed training objective caused by class imbalance, and (ii) unstable evolution of adversarial distributions. Furthermore, we show that perturbations can simultaneously address both adversarial vulnerability and class imbalance. Based on these insights, we propose \emph{{\METHODNAME}}, a plug-and-play framework that adaptively adjusts perturbations during adversarial training. Extensive experiments demonstrate that {\METHODNAME} consistently enhances adversarial robustness and class-balance on long-tailed datasets. The code is available at \href{https://github.com/zhang-lilin/RobustLT}{https://github.com/zhang-lilin/RobustLT}.
\end{abstract}
\section{Introduction}\label{se-introduction}

Deep neural networks (DNNs) have achieved remarkable success but remain vulnerable to adversarial examples, i.e, inputs crafted by adding imperceptible perturbations to natural samples \cite{biggio2013evasion, goodfellow2014explaining, ilyas2019adversarial, miller2020adversarial, szegedy2013intriguing}. Adversarial training is the most promising defense strategy to solve this threat \cite{bai2021recent, zhao2022adversarial}. It frames the learning process as a min-max game: adversarial examples are generated to maximize the classification loss, while the model is updated to minimize this loss, ultimately producing a model resilient to worst-case perturbations. 

However, most studies evaluate the performance of adversarial training using balanced datasets like CIFAR10 and CIFAR100 \cite{krizhevsky2009learning}, which do not reflect the class imbalance often found in real-world data. In practice, data often follow a long-tail distribution \cite{cao2019learning, lin2017focal}, where most samples belong to a small number of classes. Models trained on long-tail datasets tend to assign higher confidence to the majority (head) classes, which consequently undermines generalizability on the minority (tail) classes \cite{wang2017learning,wang-etal-2024-create}. It is challenging to solve such overconfidence issues \cite{buda2018systematic, he2009learning, japkowicz2002class}. However, in realistic adversarial settings, the attacker is not constrained by class frequencies and can deliberately target tail classes. Therefore, the effectiveness of adversarial training should be reevaluated under long-tail datasets \cite{gupta2019lvis, van2018inaturalist}.

\begin{figure}[t] 
	\centering
	\includegraphics[width=.99\linewidth]{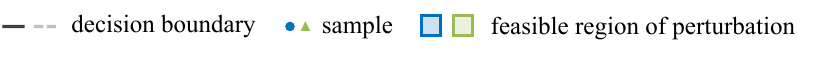} \vspace{1mm} \\
	\subfloat[Equal]{\includegraphics[width=.31\linewidth]{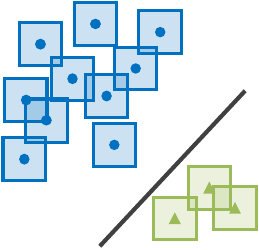} \label{fig-toy-1}} \hfil
	\subfloat[Adaptive]{\includegraphics[width=.31\linewidth]{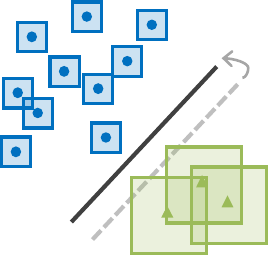} \label{fig-toy-2}} 
	\vspace{-2mm}
	\caption{Decision boundaries under equal and adaptive perturbation intensities. The perturbation intensity determines the feasible region of sample perturbation. Increasing the intensity for the minority class and reducing the intensity for the majority class shift the boundary (dashed → solid), mitigating the overconfidence.}
	\label{fig-toy}
	\vspace{-3mm}
\end{figure}

While methods like Balanced Softmax Loss (BSL) \cite{ren2020balanced} address the overconfidence issue in standard training, few have explored adversarial training under long-tail distributions. To the best of our knowledge, only a handful of works \cite{cholong, wu2021adversarial, yu2025taet, yue2024revisiting} focus on this intersection, most benefiting from combining adversarial training with BSL. However, these methods overlook a critical aspect: the interdependency between model updates and adversarial example generation. Overconfident models can bias the adversarial examples they generate, exacerbating the imbalance. Since perturbations in adversarial examples inherently alter the training distribution, a natural question arises: \textit{can they also help balance the long-tail distribution?} Intuitively, this can be achieved by assigning a higher (lower) perturbation intensity to minority (majority) class during training, as illustrated in \cref{fig-toy}.  

To gain a deeper insight, we theoretically investigate this question and find that \textit{perturbations can simultaneously address model vulnerability and overconfidence}. We first analyze the adversarial training from an online optimization perspective \cite{shi2024closer}, deriving a bound on adversarial robustness. This reveals that traditional methods struggle with long-tail data due to the skewed training objective and unstable evolution of adversarial distributions. To offset the skew in training objective, we study it in a widely-adopted binary classification case \cite{xu2021robust}, the conclusion of which suggests that class-balanced adversarial robustness can be achieved by: (i) assigning well-designed perturbation intensity for each class; (ii) controlling the shifts in adversarial distributions over iterations. This provides the theoretical foundation for using perturbations not only to improve adversarial robustness but also mitigate the negative effect of data imbalance. 

Based on the theoretical foundation, we propose \emph{{\METHODNAME}} to search adaptive perturbations during adversarial training. Different from existing works, {\METHODNAME} is a plug-and-play approach designing for long-tail scenarios that compatible with various adversarial training algorithms. {\METHODNAME} consists of (i) Class-wise Perturbation Balancing (CPB) which adjusts the perturbation adaptively for each class to restore the balanced training objective from the skewed one, and (ii) Adversarial Iteration Weighting (AIW), which stabilizes the evolution of adversarial distributions via setting an adaptive weight for each iteration to control the perturbation intensity. Extensive experiments on long-tail datasets demonstrate that {\METHODNAME} significantly improves adversarial robustness and class-balance. Our contributions are summarized as follows:
\begin{itemize}[leftmargin=1em]
\item We theoretically reveal that the impact of class imbalance arises from two key factors: a skewed training objective and unstable evolution of adversarial distributions over iterations. Meanwhile, we prove that the skew in objective can be rebalanced by well-designed perturbations.
\item We prove that the adversarial perturbations can be generated to address both adversarial vulnerability and long-tail data imbalance, which further derive a guide for finding adaptive perturbations in long-tail scenarios. 
\item We propose {\METHODNAME} based on theoretical observations, a framework compatible with various adversarial training algorithms. It effectively mitigates overconfidence and enhances robustness in long-tail scenarios as validated by comprehensive experiments on long-tail benchmarks.
\end{itemize}

\section{Related Works}\label{se-related-work}

Previous works on adversarial training \cite{madry2017towards, wang2023better, zhang2019theoretically, zhang2025weakly, zhang2024provable} have primarily focused on balanced datasets. However, real-world data often exhibit long-tail distributions \cite{cao2019learning, lin2017focal}, where a small number of head classes dominate, and a large number of tail classes are underrepresented. This discrepancy calls for a reassessment and a redesign of adversarial training methods, with performance under long-tail training data serving as a more practical and critical evaluation criterion. While several methods have been proposed to address the fairness issues in adversarial robustness \cite{lee2024dafa, li2023wat, ma2022tradeoff, sun2023improving, wei2023cfa, xu2021robust, zhang2024towards}, they mainly focus on mitigating the disparities in intrinsic learning difficulty between classes presented in balanced datasets. To date, adversarial robustness under long-tail distributions remains an underexplored area.

To our best knowledge, a very limit number of published studies specifically targeting this setting \cite{cholong, wu2021adversarial, yu2025taet, yue2024revisiting}. Both RoBal \cite{wu2021adversarial} and AT-BSL \cite{yue2024revisiting} build on the Balanced Softmax Loss (BSL) \cite{ren2020balanced}, which modifies the standard cross-entropy loss by reweighting the logits with class frequencies and applying softmax to the reweighted logits. RoBal further incorporates a cosine classifier, class-aware margins, and KL-divergence regularization, though ablations indicate BSL is the dominant contributor. Motivated by this, AT-BSL adopts a simpler design by directly combining AT \cite{madry2017towards} with BSL, and leverages data augmentation to improve robustness in long-tail settings. More recently, \cite{cholong} introduces a self-distillation framework to AT-BSL, where a balanced teacher trained on a balanced subset (constructed by upsampling head classes and downsampling tail classes) guides the student to enhance tail-class robustness. In parallel, \cite{yu2025taet} points out that BSL methods struggle with underrepresented classes and is prone to robust overfitting, and thus proposes a two-stage scheme named TAET with initial stabilization followed by stratified equalization adversarial training.

\section{Preliminaries and problem analysis}
\label{sec-diff}

In this section, we introduce the process of adversarial training and derive a bound on  adversarial robustness. It further explains why long-tail data matters in adversarial training and highlights the inherent limitations of traditional methods focusing on balanced data. These insights offer important clues for improving adversarial robustness in long-tail scenarios. The detailed proofs are provided in \cref{sec-proof1}.

\subsection{Adversarial training process}

We begin our analysis by considering an adversarial training framework following the formulation in \cite{shi2024closer}. Let $\mathcal{S} = \mathcal{X} \times \mathcal{Y}$ denote the measurable instance space, where $\mathcal{X} \subset \mathbb{R}^{d}$ represents the feature space and $\mathcal{Y} \subset \mathbb{R}$ the label space. Let ${P}$ be a sample distribution over $\mathcal{S}$ that exhibits class imbalance, satisfying ${P}(y_1) \ge {P}(y_2) \ge \dots \ge {P}(y_{\vert \mathcal{Y} \vert })$ and $P(y_1) > P(y_{\vert \mathcal{Y} \vert})$. Meanwhile, we define $\bar{P}$ as the balanced distribution, where $\bar{P}(y_i) = 1 / \vert \mathcal{Y} \vert $ and $\bar{P}(x \vert y_i) = {P}(x \vert y_i)$ for all $y_i \in \mathcal{Y}$. We use $\Vert \cdot \Vert_p$ to denote $\normlp$ norm. Let $\mathcal{H}$ be a hypothesis class consisting of functions $h: \mathcal{X} \to \mathcal{Y}$, assumed to be Lipschitz continuous, i.e., $\vert h (x_1) - h(x_2) \vert \le \rho \Vert x_1 - x_2 \Vert_1$ for some constant $\rho \ge 0$. The loss function $\loss: \mathcal{Y} \times \mathcal{Y} \to \R^+$ is assumed to be Hölder continuous, satisfying $\vert \loss (h(x_1), y_1) - \loss (h(x_2), y_2) \vert \le {c_1} (\vert h(x_1) - h(x_2) \vert + \vert y_1 -y_2 \vert)^{c_2}$ for constants ${c_1} > 0$ and $0 \le {c_2} \le 1$.

Then, we introduce two central quantities that characterize the performance of a hypothesis $h$ under distribution ${P}$: natural risk and robust risk, which represent standard generalization on natural data, and adversarial robustness in the presence of adversarial perturbations, respectively.

\begin{definition}[Natural/Robust risk]\label{de-risk}
The natural risk of a hypothesis $h \in \mathcal{H}$ under a distribution ${P}$ over $\mathcal{S}$ is defined as $\naturalrisk (h, {P}) = \mathbb{E}_{(x,y) \sim {P}} [ \loss (h(x), y)]$, and the robust risk is $\robustrisk (h, {P}) = \mathbb{E}_{(x,y) \sim {P}} [ \max_{\Vert \delta \Vert_p \le \epsilon} \loss (h(x + \delta), y)]$, where $x + \delta$ represents the adversarial example crafted against $h$ based on a clean input $x$, $\delta$ denotes the adversarial perturbation, and $\epsilon$ is a given perturbation intensity constraining the feasible region of perturbation by $\Vert \delta \Vert_p \le \epsilon$. 
\end{definition}
In this work, we consider the commonly used $\normmax$ attack and denote $\Vert \delta \Vert = \Vert \delta \Vert_\infty$ throughout the rest for simplicity. 

We can find that adversarial data can be seen as the pushforward of the original sample distribution ${P}$ with a transport map that perturbs each instance $(x, y) \sim {P}$ to its adversarial counterpart $(x + \delta, y)$. This gives rise to a new distribution (namely adversarial distribution) denoted by ${P}_{\text{adv}}^h$, which depends on the given hypothesis $h$ since the perturbation $\delta$ is crafted against $h$. With this, the robust risk can be reinterpreted as the natural risk over the adversarial distribution:
\begin{equation}\label{eq-risks-realtion}
\robustrisk (h, P) = \naturalrisk (h, P_{\text{adv}}^{h}) . 
\end{equation}

The objective of adversarial training is to find a hypothesis minimizing the robust risk, which can be actually formulated as a min-max optimization problem:
\begin{equation}\label{eq-traditional-at}
\min_{h} \robustrisk (h, {P}) = \min_{h} \mathbb{E}_{(x, y) \sim P} [ \max_{\Vert \delta \Vert \le \epsilon} \loss (h(x + \delta), y) ]. 
\end{equation}
To address this, adversarial training proceeds through a nested iterative loop, alternating between adversarial example generation (inner maximization) and model parameter update (outer minimization). 

Specifically, at each iteration $t \in [T] = \{1, 2, ..., T\}$ with $T$ the total number of training iterations, the training proceeds as follows: (i) Adversarial example generation: fix the current model $h^{(t - 1)}$\footnote{$h^{(0)}$ stands for the initialized model.}, and generate adversarial perturbations by solving $\delta = \argmax_{\Vert \delta \Vert \le \epsilon} \loss (h^{(t - 1)} (x + \delta), y)$, which results in the adversarial distribution $P_{\text{adv}}^{h^{(t - 1)}}$. Since $P_{\text{adv}}^{h^{(t - 1)}}$ considers the worse-case of $h^{(t - 1)}$, for adversarial distribution $P_{\text{adv}}^{h}$ against any other hypothesis $h \in \mathcal{H}$, it obviously follows:
\begin{equation}\label{eq-max}
\naturalrisk (h^{(t - 1)}, P_{\text{adv}}^{h^{(t - 1)}}) \ge \naturalrisk (h^{(t - 1)}, P_{\text{adv}}^{h}).
\end{equation}
(ii) Model update: update the model via gradient descent to minimize the risk on the generated adversarial distribution $P_{\text{adv}}^{h^{(t - 1)}}$, yielding $h^{(t)}$. This leads to:
\begin{equation}\label{eq-min}
\naturalrisk (h^{(t)}, P_{\text{adv}}^{h^{(t - 1)}}) \le \naturalrisk (h^{(t - 1)}, P_{\text{adv}}^{h^{(t - 1)}}). 
\end{equation}
After $T$ iterations, the final output model $h^{(T)}$ is obtained, robust risk of which measures the adversarial robustness conferred by the training algorithm. Specifically, the lower final robust risk, the higher adversarial robustness.

\subsection{Bound of adversarial robustness}
For further analysis, we clarify the factors influencing adversarial robustness here. Specifically, we derive an upper bound on the final robust risk $\mathcal{R}_{\text{rob}}(h^{(T)}, P)$ in terms of the per-iteration training objectives. 

Taking a deeper look into the adversarial training process, we can find that at each iteration $t \in [T]$, the robust risk $\mathcal{R}_{\text{rob}}(h^{(t - 1)}, P)$, which depends on the hypothesis $h^{(t - 1)}$, is minimized to get the updated hypothesis $h^{(t)}$. Consequently, each iteration optimizes a different objective. To facilitate the analysis on the relationship among the different objectives, we recap the Wasserstein distance \cite{villani2008optimal}, which quantifies the distance between two probability distributions:
\begin{definition}[Wasserstein distance]\label{de-wasserstein}
Let $m(\cdot, \cdot)$ be a metric on instance space $\mathcal{S}$. For $q \ge 1$, the $q$-Wasserstein distance between two distributions $Q_1$ and $Q_2$ over $\mathcal{S}$ is 
\begin{equation*}
\mathcal{W}_{q} (Q_1, Q_2) = \inf_{\gamma \in \Gamma (Q_1, Q_2)} \{ \mathbb{E}_{(z_1, z_2) \sim \gamma} \ m(z_1, z_2)^q \}^{1/q}, 
\end{equation*}
where $\Gamma (Q_1, Q_2)$ is the set of all couplings of $Q_1$ and $Q_2$. A coupling $\gamma$ is a joint probability distribution over $\mathcal{S} \times \mathcal{S}$ whose marginals are $Q_1$ and $Q_2$ on the first and second factors, respectively. 
\end{definition}
 The metric $m(z_1, z_2) = \Vert x_1 - x_2 \Vert_1 + \vert y_1 - y_2 \vert$ in our case, where $z_1 = (x_1, y_1)$ and $z_2 = (x_2, y_2)$. If $Q_2$ is the adversarial counterpart of $Q_1$, the distance simplifies to $\mathcal{W}_{q} (Q_1, Q_2)= \inf_{\gamma \in \Gamma (Q_1, Q_2)} \{ \mathbb{E}_{((x_1, y), (x_2, y)) \sim \gamma} \Vert x_1 - x_2 \Vert_1^q \}^{1/q}$ since the perturbation affects only the input but not the label. 
 
 We now can present the upper bound by as follow. 
\begin{theorem}[Upper bound of the final robust risk]\label{th-bound}
For any set $\{r_{t}\}_{t \in [T]}$ satisfying $r_{t} \ge 0$ and $\sum_{t = 1} ^{T}  r_{t} = T$, with $R_{t} = \sum_{t^\prime=1}^{t} r_{t^\prime}$, the adversarial training process satisfies 
\begin{equation*}
\begin{aligned}
&\robustrisk (h^{(T)}, P) \le \frac{1}{T} \sum_{t = 1} ^{T} r_{t} \naturalrisk (h^{(t)}, P_{\text{adv}}^{h^{(t - 1)}}) \\
&+ \frac{1}{T} \sum_{t = 1} ^{T} R_{t} {c_1} \{ \sqrt{{\rho^2 + 1}} \, \mathcal{W}_{{c_2} + q} (P_{\text{adv}}^{h^{(t)}}, P_{\text{adv}}^{h^{(t - 1)}}) \}^{\frac{{c_2} + q}{q}}.
\end{aligned}
\end{equation*}
\end{theorem}
\begin{remark}
\cref{th-bound} reveals that the adversarial robustness consists of two contributing terms: (i) The cumulative natural risk on adversarial distributions from previous iterations $\naturalrisk (h^{(t)}, P_{\text{adv}}^{h^{(t - 1)}})$ over iterations. Since adversarial training aims to minimize the robust risk $\robustrisk (h^{(t - 1)}, P) = \naturalrisk (h^{(t - 1)}, P_{\text{adv}}^{h^{(t - 1)}})$, this component captures the direct optimization result at each iteration. (ii) The cumulative distributional shift between adjacent adversarial distributions. It becomes large when there are significant changes in the distribution of adversarial examples across iterations, as quantified by the Wasserstein distance. 
\end{remark}

\subsection{Limitations of traditional adversarial training}
\label{sec-moti}

Adversarial training techniques have predominantly been developed under balanced distributions, which rarely hold in real-world scenarios. This discrepancy raises concerns about the practical deployments of adversarial training. 

According to \cref{th-bound}, the robust risk over the class-balanced sample distribution $\bar{P}$ can be upper bounded as:
\begin{equation}\label{eq-bal-bound}
\begin{aligned}
&\robustrisk (h^{(T)}, \bar{P}) \le \frac{1}{T} \sum_{t = 1} ^{T} r_{t} \naturalrisk (h^{(t)}, {P}_{\text{adv}}^{h^{(t - 1)}}) \\
&+ \frac{1}{T} \sum_{t = 1} ^{T} r_{t} \big( \naturalrisk (h^{(t)}, \bar{P}_{\text{adv}}^{h^{(t - 1)}}) - \naturalrisk (h^{(t)}, {P}_{\text{adv}}^{h^{(t - 1)}}) \big)\\
&+ \frac{1}{T} \sum_{t = 1} ^{T} R_{t} {c_1} \{ \sqrt{{\rho^2 + 1}} \, \mathcal{W}_{{c_2} + q} (\bar{P}_{\text{adv}}^{h^{(t)}}, \bar{P}_{\text{adv}}^{h^{(t - 1)}}) \}^{\frac{{c_2} + q}{q}}.
\end{aligned}
\end{equation} 
This bound (right hand side of \cref{eq-bal-bound}) reveals that traditional methods struggle to long-tail data due to two properties: (i) \textit{Skewed training objective}: they use the direct but imbalanced objective $\robustrisk (h^{(t - 1)}, P)$, but do not account for the skew caused by class imbalance in their training objectives, i.e, $\robustrisk (h^{(t - 1)}, \bar{P}) - \robustrisk (h^{(t - 1)}, P)$, which leads to the minimization on the first term but implicit enlargement on the second term. The skew in training objective can be rewritten as a sum over the disparities between conditional robust risks (see \cref{proof-decompose} for the derivations):
\begin{equation}\label{eq-gap}
\begin{aligned}
&\robustrisk (h^{(t - 1)}, \bar{P}) - \robustrisk (h^{(t - 1)}, {P})\\
=& \sum_{i = 2}^{\vert \mathcal{Y} \vert} \big( \frac{1}{\vert \mathcal{Y} \vert} - P(y_i) \big) \big(\robustrisk(h^{(t - 1)}, y_i) -  \robustrisk(h^{(t - 1)}, y_1) \big)
\end{aligned}
\end{equation}
where $\robustrisk (h, y) = \mathbb{E}_{x  \sim {P} (x \vert y)} [ \max_{\Vert \delta \Vert \le \epsilon} \loss (h(x + \delta), y)]$. Therefore, the skew essentially stands for the uneven robustness across classes, a pattern of overconfidence. (ii) \textit{Unstable adversarial distribution:} they do not control the evolution of adversarial distributions across iterations, which results in large distances shown in the third term. These properties lead to performance degradation on long-tail datasets even through they perform well on balanced datasets.

\section{Theoretical insights}
\label{sec-effect}
The analysis before reveals that the skewed training objective will result in an uneven robustness across classes. To deal with this, since the adversarial perturbations affect the training distribution in adversarial training, we study whether they can generated to rebalance the training process, conquering both adversarial vulnerability and long-tail data imbalance. For notational simplicity, we let $\mathcal{R} (\cdot) = \mathcal{R} (\cdot, P)$ and $\bar{\mathcal{R}} (\cdot) = \mathcal{R}(\cdot, \bar{P})$ for $\mathcal{R} \in \{ \naturalrisk, \robustrisk\}$ in the following. The detailed proofs and the scope of the theories are provided in \cref{sec-proof2} and \cref{sec-discussion}, respectively.

\subsection{A conceptual binary classification task}
As \cref{eq-gap} indicates that the model overconfidence issue can be decomposed into multiple binary imbalanced cases, we instantiate the general classification task in \cref{sec-diff} as a conceptual binary classification task. Robust and non-robust features are known to coexist in real-world data \cite{tsipras2018robustness, ilyas2019adversarial}. Following \cite{xu2021robust}, we assume that the feature space comprises: (i) robust features with center $\mu_1$ and dimension $d_1$, and (ii) non-robust features with center $\mu_2$ and dimension $d_2$, where $0 < \mu_2 < \mu_1$ and $d_1 + d_2 = d$. A data-label pair $(x, y) \sim P$ follows:
\begin{equation*}
y \sim \{+1, -1\},
\theta = ( \underbrace{\mu_1,...,\mu_1}_{\text{dim}=d_1}, \underbrace{\mu_2, ..., \mu_2}_{\text{dim}=d_2}), 
x \sim \mathcal{N}(y\theta, \sigma^2 I),
\end{equation*}
where $\theta \in \mathbb{R}^d$ is the center vector of data, and $I \in \mathbb{R}^{d \times d}$ is the identity matrix. To simulate class imbalance, we assume the class frequencies satisfy $\probability \{ y = +1\} = K  \probability \{ y = -1\}$ where $K > 1$ is the imbalance ratio. 

The hypothesis class $\mathcal{H}$ is instantiated as a set of linear classifier $h$:
\begin{equation*}
h(\cdot) = \sign (\langle w, \cdot \rangle + b), 
\end{equation*}
where $w \in \mathbb{R}^d$ is the weight vector (satisfying $w_i \ge 0$, $\Vert w \Vert_2 = 1$), and $b \in \mathbb{R}$ is the bias term. Considering 0-1 loss $\loss(h(x), y) = \indicator (h(x) \ne y) = \indicator (yh(x) < 0)$, the natural/robust risk of a classifier $h$ becomes the probability of misclassifying a clean/adversarial example: $\naturalrisk(h) =  \probability \{yh(x) < 0\} $; 
$\robustrisk(h) = \probability \{ \min_{\Vert \delta \Vert \le \epsilon} y h(x + \delta) < 0 \}$\footnote{The risks are subjected to $(x,y) \sim P$, which is omitted for simplicity.}. We assume that the perturbation intensity $\epsilon$ for adversarial robustness evaluation satisfies $\mu_2 < \epsilon< \mu_1$, meaning that an adversary can, on average, flip the sign of a non-robust feature but cannot flip robust features. This setting captures the distinction between robust and non-robust features in terms of their vulnerability to adversarial manipulation.

During adversarial training, the perturbations affect the adversarial distribution and consequently affects the training distribution. Motivated by this, we study whether adversarial perturbations can be generated to rebalance the training process in long-tail scenarios. For class $y \in \mathcal{Y}$, we assign the perturbation intensity (for training) to $\epsilon_{y}$, which makes the conditional robust risk for training becomes 
$\probability \{ \min_{\Vert \delta \Vert \le \epsilon_{y}} y h(x + \delta) < 0 \vert y\}$. 
For clarity, we define the index sets $G_1 = \{1, 2, ..., d_1\}$ and $G_2 = \{d_1+1, ..., d\}$ corresponding to robust and non-robust features, respectively, with $w_{G_1}$ and $w_{G_2}$ the components of the weight vector $w$ corresponding to them. 

\subsection{Perturbations for adversarial robustness}
Since adversarial robustness is inversely proportional to robust risk, $\robustrisk(h^{(t)}, y) \le \robustrisk(h^{(t - 1)}, y)$ serves as the criterion for determining whether the perturbations are helpful in boosting the adversarial robustness of class $y$. By this, we study the feasible region of perturbation leading to better adversarial robustness, resulting in the following theorem. 

\begin{theorem}\label{th-robustness}
For $y \in \mathcal{Y}$, if $\epsilon_{y} \in (\mu_2, \mu_1)$, it holds for $\forall t \in [T]$ that $w_{i}^{(t)} \ge w_{i}^{(t - 1)} $ for $\forall i \in G_1$, $w_{j}^{(t)} \le w_{j}^{(t - 1)} $ for $\forall j \in G_2$, and $\robustrisk(h^{(t)}, y) \le \robustrisk(h^{(t - 1)}, y)$. 
\end{theorem}
\begin{remark}
\cref{th-robustness} shows that during adversarial training, as long as perturbation are within the range $\Vert \delta \Vert \le \epsilon_{y} \text{ s.t. } \epsilon_{y} \in (\mu_2, \mu_1)$, the hypothesis will gradually rely on robust features for classification while discarding non-robust features, as the weights for robust features $w_{i}^{(t)}$ are increasing while those of non-robust features $w_{j}^{(t)}$ are decreasing. Meanwhile, since $\robustrisk(h^{(t)}, y) \le \robustrisk(h^{(t - 1)}, y)$, the hypothesis updated to be more adversarially robust.
\end{remark}

\subsection{Perturbations for class-balance}
Since the disparities $\robustrisk(h; -1) - \robustrisk(h; +1)$ and $\naturalrisk(h; -1) - \naturalrisk(h; +1)$ stand for the overconfidence of a hypothesis $h \in \mathcal{H}$ in adversarial robustness and standard generalization, respectively, we form the following lemma to conduct a unified assessment of overconfidence.  
\begin{lemma}[Indicator of overconfidence]\label{le-bias}
For an arbitrary hypothesis $h \in \mathcal{H}$ and $\mathcal{R} \in \{ \naturalrisk, \robustrisk\}$, $\sign( \mathcal{R} (h; -1) - \mathcal{R} (h; +1) ) = \sign (b)$.
\end{lemma}
\begin{remark}
\cref{le-bias} shows that the bias term plays a critical role in discriminating whether the hypothesis produces imbalanced predictions w.r.t. both standard generalization and adversarial robustness. It emphasizes that no matter what training objective the hypothesis $h^{(t)}$ is optimized for, the impact of the data imbalance will be reflected in the decision boundary of $h^{(t)}$ by its bias term. 
\end{remark}

We now search for the feasible region of perturbation leading to a more class-balanced performance.
\begin{theorem}\label{th-imbalance}
If $\epsilon_{+1} < A - \sqrt{2} \sigma \Vert w^{(t - 1)} \Vert_1^{-1} \sqrt{\log K}$ and $\epsilon_{-1} = A - \sqrt{ (A - \epsilon_{+1})^2 - 2\sigma^2 \Vert w^{(t - 1)} \Vert_1^{- 2} \log K }$, the optimal $b^{(t)}= 0$ and $\mathcal{R}(h^{(t)}, -1) - \mathcal{R}(h^{(t)}, +1) = 0$ for $\mathcal{R} \in \{ \naturalrisk, \robustrisk\}$. 
\end{theorem}
\begin{remark}
\cref{th-imbalance} shows that, by carefully assigning perturbation intensity for each class (equivalently, setting an appropriate class-wise feasible region for perturbation), adversarial training can offset the skew in the training objective and eliminate model overconfidence, as evidenced by the balanced conditional natural and robust risks.
\end{remark}

\subsection{Perturbation for both adversarial robustness and class-balance}
Now, we show that the adversarial robustness and class-balance can be simultaneously satisfied. Let $\mathcal{F}_{\text{rob}}$ be a mapping function outputs the feasible intensity of perturbation in \cref{th-robustness}, and similarly $\mathcal{F}_{\text{bal}}$ be the mapping function outputs that in \cref{th-imbalance}. 
\begin{theorem}\label{th-both}
$\mathcal{F}_{\text{rob}} (\epsilon_{+1}) \cap \mathcal{F}_{\text{bal}} (\epsilon_{+1})
= (\mu_2, A - \sqrt{2} \sigma \Vert w^{(t - 1)} \Vert_1^{-1} \sqrt{\log K}) \ne \emptyset $ and $\mathcal{F}_{\text{rob}} (\epsilon_{-1}) \cap \mathcal{F}_{\text{bal}} (\epsilon_{-1}) = \big\{ A - \sqrt{ (A - \epsilon_{+1})^2 - 2\sigma^2 \Vert w^{(t - 1)} \Vert_1^{- 2} \log K } \big\} \ne \emptyset$.
\end{theorem}
\begin{remark}
\cref{th-both} reveals that the adversarial vulnerability and overconfidence can be simultaneously satisfied, as there exist non-empty intersections between the feasible intensities for adversarial robustness (\cref{th-robustness}) and class-balance (\cref{th-imbalance}), which further guides us how to generate the desire perturbations by redesigning perturbation intensity. Since we can drive a bound for $\epsilon_{-1}$ as $\epsilon_{-1} \in \big[ \epsilon_{+1}, \epsilon_{+1} + \sqrt{2} \sigma \Vert w^{(t - 1)} \Vert_1^{-1} \sqrt{\log K} \big]$ (see \cref{proof-bound} for detail), a higher perturbation intensity should be allocated to minority class compared to majority one, where the difference between them is related to the imbalance ratio between them, i.e., $(\epsilon_{-1} - \epsilon_{+1}) \propto \sqrt{\log K}$. 
\end{remark}

In conclusion, achieving high class-balanced adversarial robustness requires the adversarial training satisfies two key properties: (i) \textit{Balanced training objective}: the hypothesis should be updated to minimize $\balrobustrisk (h^{(t - 1)})$, which can be surrogated by minimizing $\robustrisk (h^{(t - 1)})$ with deliberately designed class-wise perturbation intensity. (ii) \textit{Stable adversarial distribution}: the adversarial examples generated against $h^{(t)}$ should induce a distribution $\bar{P}_{\text{adv}}^{h^{(t)}}$ that is close (in terms of Wasserstein distance $\mathcal{W}_{{c_2}+q} $) to the one generated in the previous iteration $\bar{P}_{\text{adv}}^{h^{(t - 1)}}$.

\section{Methodology}

Based on the theoretical foundations, we propose \emph{{\METHODNAME}} to adaptively adjust perturbations across both classes and iterations via controlling the perturbation intensity. {\METHODNAME} comprises two key components: Class-wise Perturbation Balancing (CPB) and Adversarial Iteration Weighting (AIW).

\subsection{Class-wise perturbation balancing}

CPB is designed to assign adaptive perturbation intensities to each class, thereby mitigating the skew in the training objective. Let $\epsilon$ denote the commonly used perturbation intensity (e.g., $8/255$ \cite{croce2020robustbench}), and define $K_{y_i} = {P(y_1)} / {P(y_i)}$ as the class-wise imbalance ratio, where $y_1$ is the most frequent class. With $\alpha$ and $\tau$ as hyper-parameters controlling the base level and slope of the perturbation intensity across classes, the perturbation intensity for class $y_i$ is defined as:
\begin{equation}\label{eq-cw-eps}
\epsilon_{y_i} = (1 - \alpha) \epsilon + \tau \sqrt{ \log K_{y_i}} \epsilon.
\end{equation} 

To control the magnitude of perturbations, we constrain it in a distributional sense inspired by DRO \cite{staib2017distributionally} by ensuring $\mathcal{W}_{\infty}(P, P_{\text{adv}}^{h}) \leq \epsilon$. Since $\mathcal{W}_{\infty}(P, P_{\text{adv}}^{h})$ can be upper-bounded as:
\begin{equation}\label{eq-dis-inf}
\begin{aligned}
\mathcal{W}_{\infty} (P, P_{\text{adv}}^{h}) 
\le \mathbb{E}_{(x, y) \sim P} \Vert x - x_{\text{adv}} \Vert 
\le \mathbb{E}_{(x, y) \sim P} [\epsilon_{y}], 
\end{aligned}
\end{equation}
where $x_{\text{adv}}$ is the adversarial counterpart of $x$, we instead set $\mathbb{E}_{(x, y) \sim P} [\epsilon_{y}] = \epsilon$. This constraint allows us to express $\tau$ as a function of $\alpha$, namely $\tau = \alpha / \mathbb{E}_{(x, y) \sim P} [ \sqrt{ \log K_{y} }]$. Thus, CPB becomes
\begin{equation}\label{eq-eps}
\epsilon_{y} = (1 - \alpha) \epsilon + \alpha \frac{\sqrt{ \log K_{y}}}{\sum_{y^\prime \in \mathcal{Y}} [P(y^\prime) \sqrt{ \log K_{y^\prime} }] } \epsilon, 
\end{equation} 
where the hyper-parameter $\alpha \in [0, 1]$ controls the balance between the base level and the slope of the intensity across classes. A larger $\alpha$ yields a more imbalanced intensity distribution, placing greater emphasis on minority classes.

\subsection{Adversarial iteration weighting}

AIW aims at stabilizing the evolution of adversarial distributions, i.e., controlling the distance $\mathcal{W}_{c_2 + q} (\bar{P}_{\text{adv}}^{h^{(t+1)}}, \bar{P}_{\text{adv}}^{h^{(t)}} )$ for all $t \in [T]$. As $t$ becomes larger in the training process, the model approaches convergence, implying that $h^{(t)}$ is closer to $h^{(t - 1)}$, and thus the distance becomes smaller. In contrast, during the early stages of training, $h^{(t)}$ is far from $h^{(t - 1)}$, and the distance is larger. Based on this, we propose AIW to constrain this distance in the early stage. 

Specifically, to derive a general strategy regardless of $c_2 + q$, we upper bound the distance by $\mathcal{W}_{c_2 + q} \leq \mathcal{W}_{\infty}$ \cite{villani2008optimal} and triangle inequality, resulting in $\mathcal{W}_{c_2 + q} (\bar{P}_{\text{adv}}^{h^{(t+1)}}, \bar{P}_{\text{adv}}^{h^{(t)}} ) \le \mathcal{W}_{\infty} (\bar{P}, \bar{P}_{\text{adv}}^{h^{(t)}} ) + \mathcal{W}_{\infty} (\bar{P}, \bar{P}_{\text{adv}}^{h^{(t+1)}} )$. Further, let $\epsilon_{y}^{(t)}$ be the perturbation intensity for class $y \in \mathcal{Y}$ at the $t$-th iteration, the distance can be upper bounded by the perturbation intensities used in corresponding iterations according to \cref{eq-dis-inf}. Therefore, $\mathcal{W}_{c_2 + q} (\bar{P}_{\text{adv}}^{h^{(t+1)}}, \bar{P}_{\text{adv}}^{h^{(t)}} ) \le \mathbb{E}_{(x, y) \sim \bar{P}} [\epsilon_{y}^{(t)} + \epsilon_{y}^{(t+1)}]$ and we set an adaptive weight for increasing the perturbation intensity gradually from $0$ to $\epsilon_{y}$ over the first $\beta T$ iterations. This results in 
\begin{equation}\label{eq-curriculum}
\epsilon^{(t)}_y = \min \{ \frac{t - 1}{\beta T}, 1\} \cdot \epsilon_{y}. 
\end{equation}
Hyper-parameter $\beta$ controls the fraction of iterations during which the perturbation intensity gradually increases.

\subsection{Adaptive adversarial intensity}

Combining CPB and AIW, {\METHODNAME} optimizes the model in the $t$-th iteration with adaptive perturbations by
\begin{equation}\label{eq-at}
\min_{h} \mathbb{E}_{(x,y) \sim P} \max_{\delta : \Vert \delta \Vert \le \epsilon^{(t)}_y} \loss ( h^{(t - 1)} (x + \delta), y), 
\end{equation} 
where the perturbation intensity $\epsilon^{(t)}_y$ is determined using \cref{eq-eps,eq-curriculum}. Since {\METHODNAME} does not make any assumptions about the adversarial loss $\loss ( h^{(t - 1)} (x + \delta), y)$, it is compatible with various existing adversarial training algorithms. The pseudocode and more discussions of {\METHODNAME} can be found in \cref{sec-alg} and \cref{sec-discussion}, respectively.

\section{Experiments}\label{sec-experiment}

{
\setlength{\parindent}{0pt}
\textbf{Configurations.} We evaluate our method on CIFAR10-LT, CIFAR100-LT and TinyImageNet-LT, long-tailed versions of CIFAR10/100 \cite{krizhevsky2009learning} and TinyImageNet \cite{le2015tiny} (using the subset of first 20 classes) constructed using the procedure in \cite{cao2019learning} (randomly generated three times). For the main experiments, we use WRN-28-10 \cite{zagoruyko2016wide} with imbalance ratios of 50 for CIFAR10-LT and 10 for CIFAR100-LT and TinyImageNet-LT. The hyper-parameter settings of {\METHODNAME} are presented in \cref{tab-hyper-parameter} in appendix. Further settings are detailed in \cref{app-experiment}, and results across different imbalance ratios, architectures, and hyper-parameters are discussed in \cref{sec-sensitivity}.

\textbf{Evaluation metrics.} We evaluate the accuracy on natural data (\emph{Nat.}) to assess standard generalization, and adversarial robustness (\emph{Rob.}) is measured by the accuracy under $l_\infty$ PGD attack \cite{madry2017towards} with $\epsilon = 8/255$, which executed over 20 steps with a step size of $2/255$. 
Results under other attacks such as AutoAttack (AA) \cite{croce2020reliable} can be found in \cref{sec-sensitivity}. 
Since long-tailed distributions typically follow the 80/20 rule, we report the average accuracies: (i) across all classes (\emph{all}) and (ii) over the 80\% of classes with the fewest samples (\emph{tail}).

\textbf{Comparison methods.} We take existing adversarial training algorithms as base algorithms and compare the performance of them before and after equipping with {\METHODNAME}. For the base algorithms, we consider the traditional ones including AT \cite{madry2017towards} and AWP \cite{wu2020adversarial}, and methods for long-tail data including RoBal \cite{wu2021adversarial}, REAT \cite{li2023alleviating}, AT-BSL\footnote{Following \cite{yue2024revisiting}, AT-BSL-AuA is used on WRN-28-10 and AT-BSL-RA is used on ResNet-18, with suffixes omitted for simplicity.} \cite{yue2024revisiting} and TAET \cite{yu2025taet}. We further compare {\METHODNAME} with other enhancement methods, including UDR \cite{bui2022unified}, CFA \cite{wei2023cfa} and DAFA \cite{lee2024dafa}. 
}

\begin{table*}[t]
\vspace{-2mm}
	\renewcommand{\arraystretch}{}
	\centering
	\caption{Natural and robust accuracies of various base algorithms with different enhancement methods using WRN-28-10 on CIFAR10-LT, CIFAR100-LT and TinyImageNet-LT. 
	The \best{${1^{st}}$} and \second{${2^{nd}}$} results are highlighted.}
	\label{tab-main}
	\setlength{\tabcolsep}{0.15cm}
	\resizebox{\textwidth}{!}{
	\begin{tabular}{llcccccccccccc}
	\toprule
	\multirow{2.5}{*}{\makecell[c]{Base}}&\multirow{2.5}{*}{Method}{}&\multicolumn{4}{c}{CIFAR10-LT}&\multicolumn{4}{c}{CIFAR100-LT}&\multicolumn{4}{c}{TinyImagenet-LT}\\
	\cmidrule(lr{0pt}){3-6}	\cmidrule(lr{0pt}){7-10}\cmidrule(lr{0pt}){11-14}
	{}&{}&{Nat. (all)}&{Nat. (tail)}&{Rob. (all)}&{Rob. (tail)}&{Nat. (all)}&{Nat. (tail)}&{Rob. (all)}&{Rob. (tail)}&{Nat. (all)}&{Nat. (tail)}&{Rob. (all)}&{Rob. (tail)}\\
	\midrule
	\multirow{5}{*}{AT}&{origin}
	&{{58.25}${}_{\pm 0.49}$}&{{48.56}${}_{\pm 0.63}$}
	&{{27.28}${}_{\pm 0.11}$}&{{13.71}${}_{\pm 0.07}$}
	&{{44.33}${}_{\pm 0.24}$}&{{39.93}${}_{\pm 0.27}$}
	&{{15.29}${}_{\pm 0.16}$}&{{13.28}${}_{\pm 0.27}$}
	&{{44.20}${}_{\pm 0.42}$}&{{38.12}${}_{\pm 0.27}$}
	&{\second{19.50}${}_{\pm 0.49}$}&{{14.04}${}_{\pm 0.60}$}
	\\	
	{}&{UDR}
	&{{58.14}${}_{\pm 0.40}$}&{{48.42}${}_{\pm 0.45}$}
	&{{27.28}${}_{\pm 0.08}$}&{{13.72}${}_{\pm 0.14}$}
	&{{43.92}${}_{\pm 0.18}$}&{{39.43}${}_{\pm 0.17}$}
	&{{17.03}${}_{\pm 0.26}$}&{{15.02}${}_{\pm 0.23}$}
	&{{43.63}${}_{\pm 0.70}$}&{{37.42}${}_{\pm 0.82}$}
	&{{19.47}${}_{\pm 0.34}$}&{{14.12}${}_{\pm 0.51}$}
	\\
	{}&{CFA}
	&{{56.80}${}_{\pm 0.54}$}&{{46.72}${}_{\pm 0.74}$}
	&{\second{{28.50}}${}_{\pm 0.19}$ }&{{14.70}${}_{\pm 0.24}$}
	&{{43.81}${}_{\pm 0.38}$}&{{39.04}${}_{\pm 0.27}$}
	&{\second{{17.61}}${}_{\pm 0.17}$}&{\second{{15.41}}${}_{\pm 0.12}$}
	&{{44.17}${}_{\pm 0.56}$}&{{38.25}${}_{\pm 0.57}$}
	&{{17.80}${}_{\pm 0.50}$}&{{12.42}${}_{\pm 0.52}$}
	\\
	{}&{DAFA}
	&{\best{61.82}${}_{\pm 0.23}$}&{\best{53.20}${}_{\pm 0.41}$}
	&{{27.49}${}_{\pm 0.10}$}&{\second{{14.90}}${}_{\pm 0.42}$}
	&{\second{{44.64}}${}_{\pm 0.07}$}&{\second{{40.96}}${}_{\pm 0.16}$}
	&{{16.67}${}_{\pm 0.31}$}&{{15.23}${}_{\pm 0.35}$}
	&{\best{45.17}${}_{\pm 0.42}$}&{\best{40.12}${}_{\pm 0.71}$}
	&{{18.00}${}_{\pm 0.75}$}&{\second{14.62}${}_{\pm 0.35}$}
	\\
		
	\rowcolor{gray!15}
	{}&{\METHODNAME}
	&{\second{{61.59}}${}_{\pm 0.59}$}&{\second{{52.67}}${}_{\pm 0.85}$}
	&{\best{28.97}${}_{\pm 0.34}$ }&{\best{16.36}${}_{\pm 0.58}$}
	&{\best{46.73}${}_{\pm 0.20}$}&{\best{43.55}${}_{\pm 0.18}$}
	&{\best{17.64}${}_{\pm 0.17}$}&{\best{16.32}${}_{\pm 0.15}$}
	&{\second{44.57}${}_{\pm 0.12}$}&{\second{38.67}${}_{\pm 0.24}$}
	&{\best{19.57}${}_{\pm 0.21}$}&{\best{14.71}${}_{\pm 0.21}$}
	\\
	\midrule
	
	\multirow{5}{*}{AWP}&{origin}
	&{{59.66}${}_{\pm 0.21}$}&{{50.17}${}_{\pm 0.27}$}
	&{{28.50}${}_{\pm 0.38}$}&{{14.90}${}_{\pm 0.34}$}	
	&{{45.44}${}_{\pm 0.42}$}&{{40.72}${}_{\pm 0.38}$}
	&{{16.32}${}_{\pm 0.06}$}&{{14.24}${}_{\pm 0.29}$}
	&{{40.70}${}_{\pm 1.23}$}&{{34.21}${}_{\pm 1.38}$}
	&{{20.47}${}_{\pm 1.19}$}&{{14.92}${}_{\pm 1.27}$}
	\\
	{}&{UDR}
	&{{59.65}${}_{\pm 0.53}$}&{{50.10}${}_{\pm 0.69}$}
	&{{28.59}${}_{\pm 0.32}$}&{{14.92}${}_{\pm 0.32}$}
	&{{45.48}${}_{\pm 0.28}$}&{{40.75}${}_{\pm 0.30}$}
	&{{18.36}${}_{\pm 0.21}$}&{{15.98}${}_{\pm 0.20}$}
	&{{40.87}${}_{\pm 0.17}$}&{{34.50}${}_{\pm 0.62}$}
	&{\second{21.23}${}_{\pm 0.59}$}&{\second{15.54}${}_{\pm 0.66}$}
	\\
	{}&{CFA}
	&{{56.90}${}_{\pm 1.16}$}&{{46.71}${}_{\pm 1.45}$}
	&{\best{29.31}${}_{\pm 0.08}$}&{\second{{15.45}}${}_{\pm 0.15}$}
	&{{45.22}${}_{\pm 0.10}$}&{{40.28}${}_{\pm 0.18}$}
	&{\second{{19.02}}${}_{\pm 0.17}$}&{{16.48}${}_{\pm 0.35}$}
	&{\best{43.57}${}_{\pm 0.73}$}&{\best{37.29}${}_{\pm 0.62}$}
	&{{18.27}${}_{\pm 0.46}$}&{{12.96}${}_{\pm 0.29}$}
	\\
	{}&{DAFA}
	&{\second{{60.63}}${}_{\pm 0.45}$}&{\second{{51.43}}${}_{\pm 0.58}$}
	&{{28.78}${}_{\pm 0.07}$}&{{15.37}${}_{\pm 0.19}$}	
	&{\second{{46.26}}${}_{\pm 0.18}$}&{\second{{42.32}}${}_{\pm 0.19}$}
	&{{18.55}${}_{\pm 0.31}$}&{\second{{16.60}}${}_{\pm 0.30}$}
	&{\second{42.30}${}_{\pm 0.28}$}&{\second{36.25}${}_{\pm 0.47}$}
	&{{20.20}${}_{\pm 0.94}$}&{{15.29}${}_{\pm 0.77}$}
	\\
		
	\rowcolor{gray!15}
	{}&{\METHODNAME}
	&{\best{65.22}${}_{\pm 0.27}$}&{\best{57.05}${}_{\pm 0.37}$}
	&{\second{{29.11}}${}_{\pm 0.42}$ }&{\best{16.46}${}_{\pm 0.54}$}
	&{\best{50.19}${}_{\pm 0.23}$}&{\best{46.51}${}_{\pm 0.18}$}
	&{\best{19.37}${}_{\pm 0.35}$}&{\best{17.57}${}_{\pm 0.30}$}
	&{{41.83}${}_{\pm 1.25}$}&{{35.54}${}_{\pm 1.33}$}
	&{\best{21.33}${}_{\pm 0.68}$}&{\best{16.17}${}_{\pm 0.12}$}
	\\
	\midrule
	
	\multirow{5}{*}{RoBal}&{RoBal}
	&{{72.73}${}_{\pm 0.13}$}&{{67.15}${}_{\pm 0.13}$}
	&{{32.29}${}_{\pm 0.51}$}&{{22.19}${}_{\pm 0.41}$}	
	&{{49.36}${}_{\pm 0.17}$}&{{46.83}${}_{\pm 0.10}$}
	&{{16.51}${}_{\pm 0.06}$}&{{15.32}${}_{\pm 0.13}$}
	&{{49.83}${}_{\pm 0.39}$}&{\best{47.92}${}_{\pm 0.62}$}
	&{{21.50}${}_{\pm 0.29}$}&{\second{18.54}${}_{\pm 0.42}$}
	\\
	{}&{UDR}
	&{{62.38}${}_{\pm 0.57}$}&{{53.60}${}_{\pm 0.74}$}
	&{{31.11}${}_{\pm 0.50}$}&{{18.28}${}_{\pm 0.64}$}
	&{{49.00}${}_{\pm 0.58}$}&{{46.67}${}_{\pm 0.74}$}
	&{\second{{19.41}}${}_{\pm 0.15}$}&{\second{{18.28}}${}_{\pm 0.15}$}
	&{{44.47}${}_{\pm 0.33}$}&{{38.88}${}_{\pm 0.27}$}
	&{{21.23}${}_{\pm 0.52}$}&{{16.33}${}_{\pm 0.50}$}
	\\
	{}&{CFA}
	&{\second{{72.81}}${}_{\pm 0.54}$}&{\second{{67.42}}${}_{\pm 0.67}$}
	&{\second{{33.35}}${}_{\pm 0.83}$}&{\second{{23.52}}${}_{\pm 0.93}$}
	&{\best{49.84}${}_{\pm 0.32}$}&{\second{{47.35}}${}_{\pm 0.22}$}
	&{{19.16}${}_{\pm 0.17}$}&{{17.95}${}_{\pm 0.21}$}
	&{\best{50.57}${}_{\pm 0.66}$}&{\second{47.79}${}_{\pm 0.93}$}
	&{{20.30}${}_{\pm 0.83}$}&{{17.29}${}_{\pm 0.89}$}
	\\
	{}&{DAFA}
	&{{72.25}${}_{\pm 0.21}$}&{{66.78}${}_{\pm 0.38}$}
	&{{32.87}${}_{\pm 0.41}$}&{{23.22}${}_{\pm 0.55}$}
	&{{49.17}${}_{\pm 0.07}$}&{{46.68}${}_{\pm 0.04}$}
	&{{19.15}${}_{\pm 0.12}$}&{{17.91}${}_{\pm 0.13}$}
	&{{49.13}${}_{\pm 0.45}$}&{{47.33}${}_{\pm 0.39}$}
	&{\second{21.51}${}_{\pm 0.54}$}&{{18.50}${}_{\pm 0.57}$}
	\\
		
	\rowcolor{gray!15}
	{}&{\METHODNAME}
	&{\best{74.63}${}_{\pm 0.28}$}&{\best{70.32}${}_{\pm 0.32}$}
	&{\best{36.08}${}_{\pm 0.35}$}&{\best{29.19}${}_{\pm 0.19}$}
	&{\second{{49.79}}${}_{\pm 0.11}$}&{\best{47.83}${}_{\pm 0.11}$}
	&{\best{19.87}${}_{\pm 0.07}$}&{\best{19.51}${}_{\pm 0.06}$}
	&{\second{49.90}${}_{\pm 0.54}$}&{{47.38}${}_{\pm 0.31}$}
	&{\best{22.27}${}_{\pm 0.58}$}&{\best{20.08}${}_{\pm 0.31}$}
	\\
	\midrule
	
	\multirow{5}{*}{REAT}&{origin}
	&{{67.78}${}_{\pm 0.19}$}&{{60.81}${}_{\pm 0.25}$}
	&{{28.67}${}_{\pm 0.46}$}&{{17.12}${}_{\pm 0.44}$}
	&{{47.22}${}_{\pm 0.14}$}&{{45.43}${}_{\pm 0.06}$}
	&{{15.40}${}_{\pm 0.17}$}&{{14.47}${}_{\pm 0.23}$}
	&{{44.03}${}_{\pm 0.90}$}&{{41.30}${}_{\pm 0.27}$}
	&{{21.43}${}_{\pm 0.83}$}&{{19.33}${}_{\pm 1.09}$}
	\\
	{}&{UDR}
	&{{68.22}${}_{\pm 0.27}$}&{{61.38}${}_{\pm 0.38}$}
	&{{29.28}${}_{\pm 0.19}$}&{{17.60}${}_{\pm 0.28}$}
	&{{47.20}${}_{\pm 0.33}$}&{{45.13}${}_{\pm 0.36}$}
	&{{17.67}${}_{\pm 0.27}$}&{{16.65}${}_{\pm 0.37}$}
	&{{44.20}${}_{\pm 1.20}$}&{{41.46}${}_{\pm 0.87}$}
	&{{20.87}${}_{\pm 1.65}$}&{{18.58}${}_{\pm 1.90}$}
	\\
	{}&{CFA}
	&{\second{{68.68}}${}_{\pm 0.41}$}&{\second{{62.20}}${}_{\pm 0.53}$}
	&{\second{{30.66}}${}_{\pm 0.29}$}&{\second{{19.71}}${}_{\pm 0.59}$}
	&{\second{{47.84}}${}_{\pm 0.24}$}&{\second{{45.54}}${}_{\pm 0.33}$}
	&{\second{{18.30}}${}_{\pm 0.17}$}&{\second{{17.28}}${}_{\pm 0.16}$}
	&{\second{46.23}${}_{\pm 1.31}$}&{\second{43.17}${}_{\pm 1.02}$}
	&{\best{22.83}${}_{\pm 0.50}$}&{\second{19.75}${}_{\pm 0.47}$}
	\\
	{}&{DAFA}
	&{{67.84}${}_{\pm 0.22}$}&{{61.45}${}_{\pm 0.15}$}
	&{{29.32}${}_{\pm 0.09}$}&{{19.01}${}_{\pm 0.18}$}
	&{{46.91}${}_{\pm 0.32}$}&{{45.02}${}_{\pm 0.29}$}
	&{{17.52}${}_{\pm 0.25}$}&{{16.66}${}_{\pm 0.31}$}
	&{{43.57}${}_{\pm 0.78}$}&{{41.67}${}_{\pm 1.03}$}
	&{{21.33}${}_{\pm 1.57}$}&{{19.50}${}_{\pm 1.40}$}
	\\
		
	\rowcolor{gray!15}
	{}&{\METHODNAME}
	&{\best{71.98}${}_{\pm 0.14}$}&{\best{66.58}${}_{\pm 0.15}$}
	&{\best{33.01}${}_{\pm 0.33}$}&{\best{24.21}${}_{\pm 0.62}$}
	&{\best{48.74}${}_{\pm 0.32}$}&{\best{47.69}${}_{\pm 0.26}$}
	&{\best{18.33}${}_{\pm 0.06}$}&{\best{18.08}${}_{\pm 0.10}$}
	&{\best{46.40}${}_{\pm 1.44}$}&{\best{43.33}${}_{\pm 1.00}$}
	&{\second{22.57}${}_{\pm 0.47}$}&{\best{20.54}${}_{\pm 1.07}$}
	\\
	\midrule
	
	\multirow{5}{*}{AT-BSL}&{origin}
	&{{77.09}${}_{\pm 0.41}$}&{{72.48}${}_{\pm 0.30}$}
	&{{37.98}${}_{\pm 0.42}$}&{{28.60}${}_{\pm 0.85}$}
	&{{55.00}${}_{\pm 0.08}$}&{{52.75}${}_{\pm 0.12}$}
	&{{19.82}${}_{\pm 0.18}$}&{{18.79}${}_{\pm 0.25}$}
	&{{40.00}${}_{\pm 0.99}$}&{{36.96}${}_{\pm 1.12}$}
	&{\second{{25.30}}${}_{\pm 0.43}$}&{\second{{22.96}}${}_{\pm 0.48}$}
	\\
	{}&{UDR}
	&{\second{{77.66}}${}_{\pm 0.49}$}&{{73.40}${}_{\pm 0.76}$}
	&{{38.40}${}_{\pm 0.47}$}&{{29.78}${}_{\pm 0.62}$}
	&{{54.88}${}_{\pm 0.13}$}&{{52.75}${}_{\pm 0.19}$}
	&{\second{{23.39}}${}_{\pm 0.14}$}&{\second{{22.46}}${}_{\pm 0.12}$}
	&{{40.23}${}_{\pm 1.80}$}&{{37.38}${}_{\pm 2.04}$}
	&{{25.00}${}_{\pm 0.70}$}&{{22.58}${}_{\pm 0.72}$}
	\\
	{}&{CFA}
	&{\best{77.91}${}_{\pm 0.48}$}&{\second{{73.67}}${}_{\pm 0.66}$}
	&{\second{{38.89}}${}_{\pm 0.19}$}&{{30.27}${}_{\pm 0.47}$}
	&{\best{56.26}${}_{\pm 0.53}$}&{\best{54.05}${}_{\pm 0.51}$}
	&{{22.51}${}_{\pm 0.33}$}&{{21.40}${}_{\pm 0.36}$}
	&{\best{49.07}${}_{\pm 1.03}$}&{\best{46.00}${}_{\pm 1.25}$}
	&{{25.07}${}_{\pm 0.25}$}&{{22.75}${}_{\pm 0.31}$}
	\\
	{}&{DAFA}
	&{{77.31}${}_{\pm 0.36}$}&{{73.23}${}_{\pm 0.44}$}
	&{{38.11}${}_{\pm 0.35}$}&{\second{{30.30}}${}_{\pm 0.44}$}
	&{{54.52}${}_{\pm 0.21}$}&{{52.18}${}_{\pm 0.36}$}
	&{{23.00}${}_{\pm 0.18}$}&{{22.12}${}_{\pm 0.32}$}
	&{{40.30}${}_{\pm 1.56}$}&{{37.54}${}_{\pm 1.99}$}
	&{{24.80}${}_{\pm 0.36}$}&{{22.83}${}_{\pm 0.33}$}
	\\
		
	\rowcolor{gray!15}
	{}&{\METHODNAME}
	&{{77.61}${}_{\pm 0.54}$}&{\best{73.83}${}_{\pm 0.60}$}
	&{\best{42.11}${}_{\pm 0.36}$}&{\best{35.98}${}_{\pm 0.79}$}
	&{\second{{55.64}}${}_{\pm 0.21}$}&{\second{{53.49}}${}_{\pm 0.20}$}
	&{\best{23.77}${}_{\pm 0.20}$}&{\best{23.34}${}_{\pm 0.38}$}
	&{\second{{44.87}}${}_{\pm 1.35}$}&{\second{{41.42}}${}_{\pm 1.43}$}
	&{\best{27.00}${}_{\pm 0.29}$}&{\best{24.83}${}_{\pm 0.21}$}
	\\
	\midrule
	
	\multirow{5}{*}{TAET}&{origin}
	&{{65.32}${}_{\pm 2.28}$}&{{60.44}${}_{\pm 1.97}$}
	&{{31.68}${}_{\pm 0.54}$}&{{24.67}${}_{\pm 0.61}$}
	&{{45.98}${}_{\pm 0.22}$}&{{41.42}${}_{\pm 0.11}$}
	&{{14.82}${}_{\pm 0.17}$}&{{13.15}${}_{\pm 0.29}$}
	&{{31.70}${}_{\pm 1.18}$}&{{28.58}${}_{\pm 1.25}$}
	&{\second{17.33}${}_{\pm 0.41}$}&{\second{15.21}${}_{\pm 0.72}$}
	\\
	{}&{UDR}
	&{{67.12}${}_{\pm 0.49}$}&{{62.57}${}_{\pm 0.69}$}
	&{{32.22}${}_{\pm 0.17}$}&{{25.44}${}_{\pm 0.55}$}
	&{{45.91}${}_{\pm 0.06}$}&{{41.42}${}_{\pm 0.17}$}
	&{{18.77}${}_{\pm 0.17}$}&{{16.43}${}_{\pm 0.23}$}
	&{{25.40}${}_{\pm 5.90}$}&{{22.33}${}_{\pm 5.85}$}
	&{{13.57}${}_{\pm 2.56}$}&{{11.75}${}_{\pm 2.62}$}
	\\
	{}&{CFA}
	&{\best{68.83}${}_{\pm 0.40}$}&{\best{64.25}${}_{\pm 0.86}$}
	&{\second{{32.70}}${}_{\pm 0.66}$}&{{25.75}${}_{\pm 1.40}$}
	&{\second{{46.33}}${}_{\pm 0.13}$}&{\second{{41.74}}${}_{\pm 0.17}$}
	&{{18.93}${}_{\pm 0.25}$}&{\second{{16.69}}${}_{\pm 0.32}$}
	&{\second{32.57}${}_{\pm 2.56}$}&{\second{30.42}${}_{\pm 2.16}$}
	&{{13.90}${}_{\pm 1.16}$}&{{12.67}${}_{\pm 1.33}$}
	\\
	{}&{DAFA}
	&{{67.09}${}_{\pm 0.32}$}&{{62.66}${}_{\pm 0.65}$}
	&{{32.12}${}_{\pm 0.29}$}&{\second{{25.91}}${}_{\pm 0.80}$}
	&{{45.65}${}_{\pm 0.29}$}&{{41.12}${}_{\pm 0.21}$}
	&{\second{{18.94}}${}_{\pm 0.22}$}&{{16.46}${}_{\pm 0.38}$}
	&{{29.67}${}_{\pm 3.20}$}&{{27.00}${}_{\pm 3.17}$}
	&{{16.10}${}_{\pm 1.28}$}&{{14.12}${}_{\pm 1.43}$}
	\\
		
	\rowcolor{gray!15}
	{}&{\METHODNAME}
	&{\second{68.05}${}_{\pm 1.37}$}&{\second{64.36}${}_{\pm 1.13}$}
	&{\best{34.05}${}_{\pm 0.17}$}&{\best{29.33}${}_{\pm 0.42}$}
	&{\best{46.42}${}_{\pm 0.14}$}&{\best{42.66}${}_{\pm 0.07}$}
	&{\best{19.57}${}_{\pm 0.24}$}&{\best{18.01}${}_{\pm 0.22}$}
	&{\best{36.80}${}_{\pm 1.96}$}&{\best{34.58}${}_{\pm 2.15}$}
	&{\best{18.13}${}_{\pm 0.71}$}&{\best{16.96}${}_{\pm 0.82}$}
	\\
	\bottomrule
	\end{tabular}
	}
	\vspace{-2mm}
\end{table*}

\subsection{Main results}

From \cref{tab-main}, we observe that {\METHODNAME} consistently improves both natural and robust accuracies compared to original base methods, with particularly notable gains on tail-class. These results suggest that {\METHODNAME} effectively mitigates the overconfidence issue induced by long-tailed data, striking a better tradeoff between adversarial robustness and class-balance. We also compare {\METHODNAME} with other enhancement methods in \cref{tab-main}, with an intuitive visualization of their perturbation intensity shown in \cref{fig-cw-atbsl} in appendix. {\METHODNAME} clearly outperforms the alternatives towards robustness in long-tail scenarios. Additional study about the adversarial distributions is provided in \cref{sec-ae}.

\subsection{Sensitivity analysis}\label{sec-sensitivity}
\begin{figure*}[t] 
\vspace{-2mm}
	\centering
	\subfloat[Robust accuracy on WRN-28-10]{\includegraphics[width=.28\linewidth]{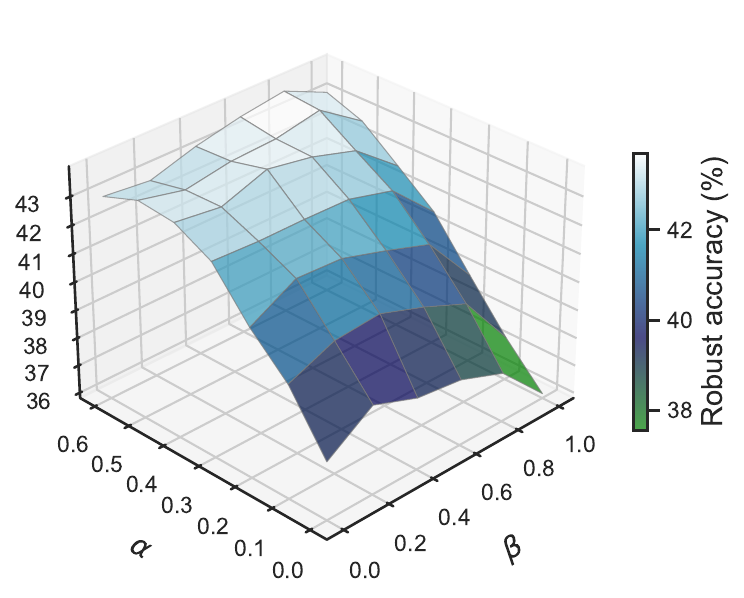} \label{fig-hyper-rob}} \hfil
	\subfloat[Natural accuracy on WRN-28-10]{\includegraphics[width=.28\linewidth]{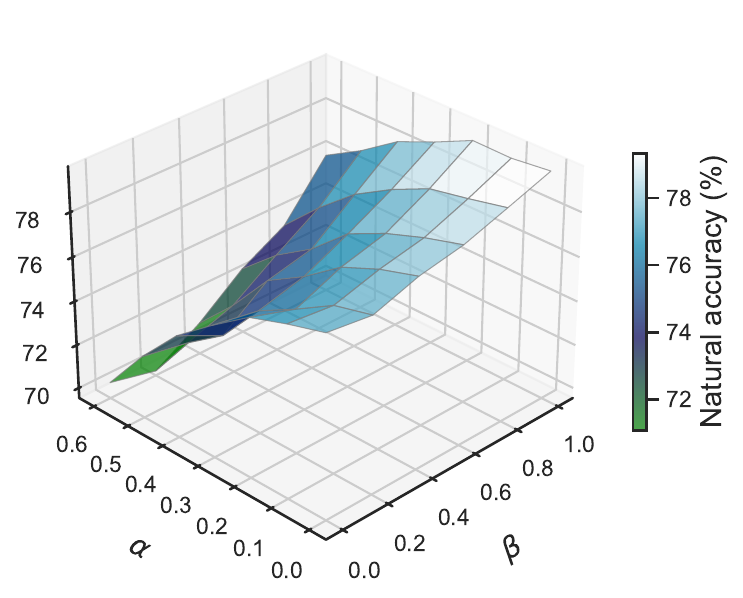} \label{fig-hyper-nat}} \hfil
	\subfloat[Tradeoff on WRN-28-10]{\includegraphics[width=.26\linewidth]{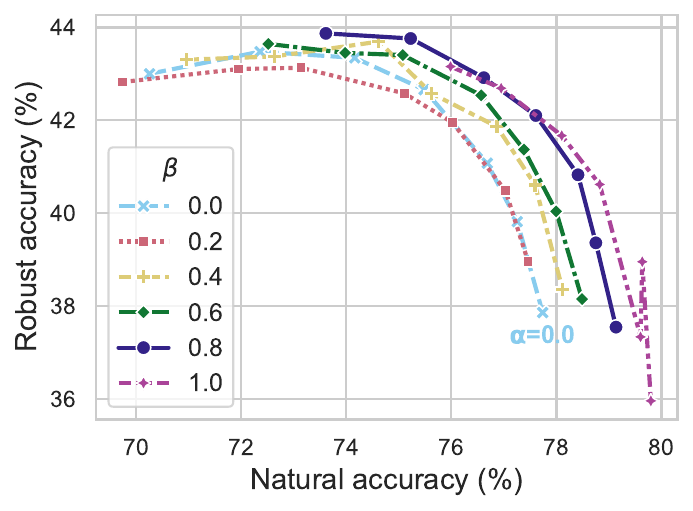} \label{fig-hyper-tradeoff}} 
	\\	
	\subfloat[Robust accuracy on ResNet-18]{\includegraphics[width=.28\linewidth]{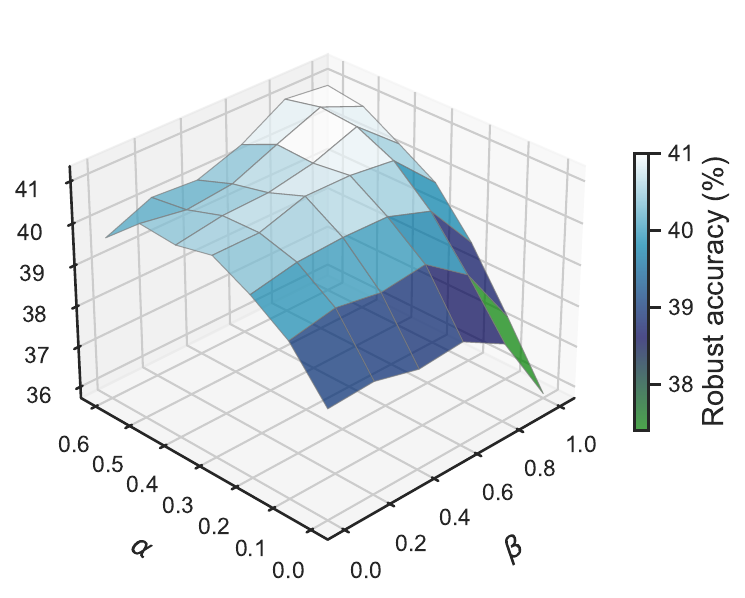} \label{fig-hyper-rob-rn}} \hfil
	\subfloat[Natural accuracy on ResNet-18]{\includegraphics[width=.28\linewidth]{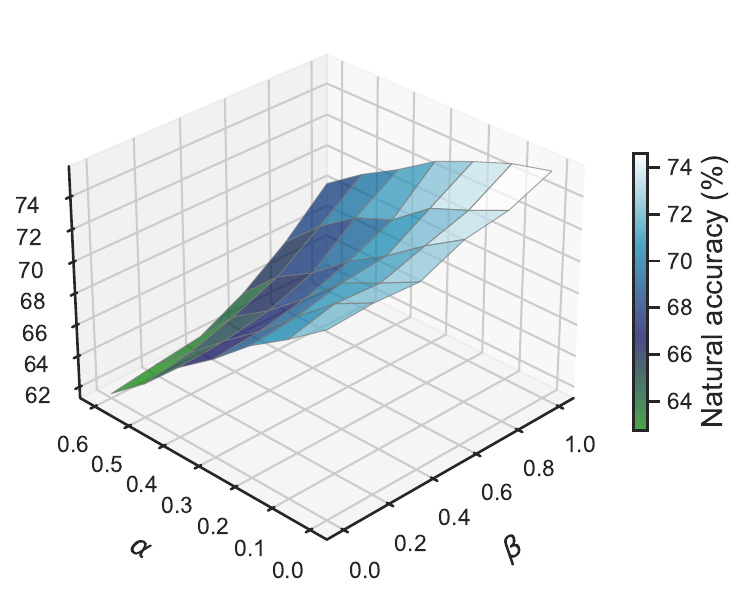} \label{fig-hyper-nat-rn}} \hfil
	\subfloat[Tradeoff on ResNet-18]{\includegraphics[width=.26\linewidth]{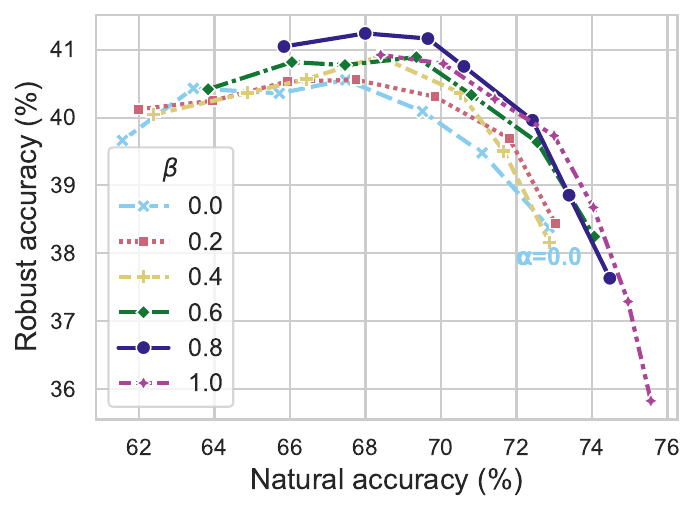} \label{fig-hyper-tradeoff-rn}} 
	\caption{Robust accuracy, natural accuracy, and the tradeoff between them under varying settings of $\alpha$ and $\beta$ when applying to AT-BSL on CIFAR10-LT with different model architectures. The tradeoff curves correspond to varying values of $\beta$, with individual points representing different $\alpha \in \{0.0, 0.1, 0.2, 0.3, 0.4, 0.5, 0.6\}$ in order.}
	\label{fig-hyper}
\end{figure*}

\begin{table*}[t]
	\renewcommand{\arraystretch}{}
	\centering
	\caption{Robust accuracies against different attacks using WRN-28-10. Better results are bolded.}
	\label{tab-attacks}
	\setlength{\tabcolsep}{0.15cm}
	\resizebox{\linewidth}{!}{
	\begin{tabular}{llcccccccccccc}
	\toprule
	\multirow{2.5}{*}{Base}&\multirow{2.5}{*}{Method}&\multicolumn{4}{c}{CIFAR10-LT}&\multicolumn{4}{c}{CIFAR100-LT}&\multicolumn{4}{c}{TinyImageNet-LT}\\
	\cmidrule(lr{0pt}){3-6}	\cmidrule(lr{0pt}){7-10}\cmidrule(lr{0pt}){11-14}
	{}&{}&\makecell[c]{CW (all)}&\makecell[c]{CW (tail)}&\makecell[c]{AA (all)}&\makecell[c]{AA (tail)}&\makecell[c]{CW (all)}&\makecell[c]{CW (tail)}&\makecell[c]{AA (all)}&\makecell[c]{AA (tail)}&\makecell[c]{CW (all)}&\makecell[c]{CW (tail)}&\makecell[c]{AA (all)}&\makecell[c]{AA (tail)}\\
	\midrule
	
	{}&{origin}
	&{{27.25}${}_{\pm 0.17}$}&{{13.58}${}_{\pm 0.17}$}
	&{{26.15}${}_{\pm 0.15}$}&{{12.42}${}_{\pm 0.18}$}
	&{{16.46}${}_{\pm 0.13}$}&{{14.32}${}_{\pm 0.25}$}
	&{{15.29}${}_{\pm 0.16}$}&{{13.28}${}_{\pm 0.27}$}
	&{{18.63}${}_{\pm 0.33}$}&{{13.21}${}_{\pm 0.39}$}
	&{{17.33}${}_{\pm 0.33}$}&{{12.33}${}_{\pm 0.24}$}
	\\
	
	\rowcolor{gray!15}
	{AT}&{\METHODNAME}
	&{\best{29.32}${}_{\pm 0.45}$}&{\best{16.60}${}_{\pm 0.74}$}
	&{\best{27.49}${}_{\pm 0.34}$}&{\best{14.83}${}_{\pm 0.62}$}
	&{\best{17.24}${}_{\pm 0.16}$}&{\best{15.93}${}_{\pm 0.23}$}
	&{\best{15.69}${}_{\pm 0.18}$}&{\best{14.43}${}_{\pm 0.24}$}
	&{\best{18.67}${}_{\pm 0.19}$}&{\best{14.08}${}_{\pm 0.16}$}
	&{\best{17.53}${}_{\pm 0.09}$}&{\best{13.04}${}_{\pm 0.04}$}
	\\
	\midrule
	
	{}&{origin}
	&{{28.64}${}_{\pm 0.36}$}&{{15.03}${}_{\pm 0.35}$}
	&{{27.50}${}_{\pm 0.36}$}&{{13.80}${}_{\pm 0.34}$}
	&{{17.50}${}_{\pm 0.11}$}&{{15.29}${}_{\pm 0.31}$}
	&{{16.32}${}_{\pm 0.06}$}&{{14.24}${}_{\pm 0.29}$}
	&{\best{19.93}${}_{\pm 0.95}$}&{{14.46}${}_{\pm 0.89}$}
	&{{18.57}${}_{\pm 0.93}$}&{{13.33}${}_{\pm 0.87}$}
	\\
	
	\rowcolor{gray!15}
	{AWP}&{\METHODNAME}
	&{\best{29.68}${}_{\pm 0.31}$}&{\best{17.00}${}_{\pm 0.43}$}
	&{\best{27.59}${}_{\pm 0.40}$}&{\best{14.93}${}_{\pm 0.50}$}
	&{\best{18.55}${}_{\pm 0.22}$}&{\best{16.89}${}_{\pm 0.26}$}
	&{\best{16.86}${}_{\pm 0.27}$}&{\best{15.28}${}_{\pm 0.31}$}
	&{{19.80}${}_{\pm 0.08}$}&{\best{14.75}${}_{\pm 0.10}$}
	&{\best{18.63}${}_{\pm 0.45}$}&{\best{13.96}${}_{\pm 0.41}$}
	\\
	\midrule
	
	{}&{origin}
	&{{31.29}${}_{\pm 0.42}$}&{{20.79}${}_{\pm 0.38}$}
	&{{28.92}${}_{\pm 0.37}$}&{{18.58}${}_{\pm 0.37}$}
	&{{18.42}${}_{\pm 0.14}$}&{{17.20}${}_{\pm 0.17}$}
	&{{16.51}${}_{\pm 0.06}$}&{{15.32}${}_{\pm 0.13}$}
	&{{19.47}${}_{\pm 0.25}$}&{{16.33}${}_{\pm 0.06}$}
	&{{18.50}${}_{\pm 0.42}$}&{{15.38}${}_{\pm 0.27}$}
	\\
	
	\rowcolor{gray!15}
	{RoBal}&{\METHODNAME}
	&{\best{33.90}${}_{\pm 0.27}$}&{\best{26.30}${}_{\pm 0.04}$}
	&{\best{31.03}${}_{\pm 0.20}$}&{\best{23.53}${}_{\pm 0.16}$}
	&{\best{18.98}${}_{\pm 0.19}$}&{\best{18.42}${}_{\pm 0.32}$}
	&{\best{16.93}${}_{\pm 0.20}$}&{\best{16.25}${}_{\pm 0.31}$}
	&{\best{20.00}${}_{\pm 0.22}$}&{\best{17.75}${}_{\pm 0.20}$}
	&{\best{18.83}${}_{\pm 0.37}$}&{\best{16.42}${}_{\pm 0.36}$}
	\\
	\midrule
	
	{}&{origin}
	&{{29.13}${}_{\pm 0.49}$}&{{17.58}${}_{\pm 0.44}$}
	&{{27.11}${}_{\pm 0.49}$}&{{15.50}${}_{\pm 0.44}$}
	&{{16.90}${}_{\pm 0.30}$}&{{15.97}${}_{\pm 0.25}$}
	&{{15.40}${}_{\pm 0.17}$}&{{14.47}${}_{\pm 0.23}$}
	&{{19.70}${}_{\pm 0.99}$}&{{17.50}${}_{\pm 1.18}$}
	&{{17.60}${}_{\pm 0.40}$}&{{15.12}${}_{\pm 0.62}$}
	\\
	
	\rowcolor{gray!15}
	{REAT}&{\METHODNAME}
	&{\best{33.15}${}_{\pm 0.33}$}&{\best{24.20}${}_{\pm 0.41}$}
	&{\best{29.98}${}_{\pm 0.17}$}&{\best{20.86}${}_{\pm 0.39}$}
	&{\best{17.72}${}_{\pm 0.08}$}&{\best{17.49}${}_{\pm 0.16}$}
	&{\best{16.00}${}_{\pm 0.14}$}&{\best{15.77}${}_{\pm 0.22}$}
	&{\best{20.37}${}_{\pm 0.78}$}&{\best{18.38}${}_{\pm 0.74}$}
	&{\best{19.30}${}_{\pm 0.67}$}&{\best{17.33}${}_{\pm 0.86}$}
	\\
	\midrule
	
	{}&{origin}
	&{{37.13}${}_{\pm 0.14}$}&{{27.58}${}_{\pm 0.51}$}
	&{{34.57}${}_{\pm 0.18}$}&{{24.84}${}_{\pm 0.55}$}
	&{{22.13}${}_{\pm 0.09}$}&{{21.13}${}_{\pm 0.16}$}
	&{{19.82}${}_{\pm 0.18}$}&{{18.79}${}_{\pm 0.25}$}
	&{{22.80}${}_{\pm 0.70}$}&{{20.21}${}_{\pm 0.66}$}
	&{{21.40}${}_{\pm 0.65}$}&{{18.83}${}_{\pm 0.75}$}
	\\

	\rowcolor{gray!15}	
	{AT-BSL}&{\METHODNAME}
	&{\best{40.10}${}_{\pm 0.39}$}&{\best{33.50}${}_{\pm 0.45}$}
	&{\best{37.49}${}_{\pm 0.25}$}&{\best{30.70}${}_{\pm 0.36}$}
	&{\best{22.51}${}_{\pm 0.24}$}&{\best{21.88}${}_{\pm 0.26}$}
	&{\best{19.92}${}_{\pm 0.26}$}&{\best{19.16}${}_{\pm 0.26}$}
	&{\best{23.77}${}_{\pm 0.34}$}&{\best{21.21}${}_{\pm 0.31}$}
	&{\best{22.37}${}_{\pm 0.21}$}&{\best{20.04}${}_{\pm 0.12}$}
	\\
	\midrule
	
	{}&{origin}
	&{{29.89}${}_{\pm 0.91}$}&{{22.80}${}_{\pm 0.30}$}
	&{{27.88}${}_{\pm 0.86}$}&{{20.77}${}_{\pm 0.38}$}
	&{{16.94}${}_{\pm 0.23}$}&{{14.88}${}_{\pm 0.29}$}
	&{{14.82}${}_{\pm 0.17}$}&{{13.15}${}_{\pm 0.29}$}
	&{{13.63}${}_{\pm 0.65}$}&{{11.79}${}_{\pm 0.74}$}
	&{{12.83}${}_{\pm 0.84}$}&{{10.96}${}_{\pm 1.03}$}
	\\

	\rowcolor{gray!15}
	{TAET}&{\METHODNAME}
	&{\best{31.86}${}_{\pm 0.29}$}&{\best{26.72}${}_{\pm 0.08}$}
	&{\best{29.64}${}_{\pm 0.17}$}&{\best{24.52}${}_{\pm 0.18}$}
	&{\best{17.64}${}_{\pm 0.21}$}&{\best{16.17}${}_{\pm 0.31}$}
	&{\best{15.49}${}_{\pm 0.12}$}&{\best{14.34}${}_{\pm 0.25}$}
	&{\best{14.50}${}_{\pm 1.02}$}&{\best{13.21}${}_{\pm 1.43}$}
	&{\best{13.57}${}_{\pm 1.15}$}&{\best{12.33}${}_{\pm 1.49}$}
	\\
	\bottomrule
	\end{tabular}
	}
	\vspace{-2mm}
\end{table*}

\setlength{\parindent}{0pt}
\textbf{Hyper-parameter.} To examine the sensitivity of {\METHODNAME} to hyper-parameters $\alpha$ and $\beta$, we evaluate model performance under varying values, as shown in \cref{fig-hyper}. The results indicate: (i) Enlarging $\alpha$ boosts robust accuracy, confirming the effectiveness of CPB to offset the skew in adversarial training objective. However, an excessively large $\alpha$ can hurt robustness, due to it beyonds the appropriate limit according to \cref{th-robustness}. (ii) Increasing $\beta$ can improve robust accuracy, verifying the effectiveness of AIW to stabilize the evolution of adversarial distributions. However, since AIW with a larger $\beta$ encourages the model to fit the adversarial distribution closer to the natural one, an excessively large $\beta$ can achieve high natural accuracy but hurt robustness. (iii) The joint contribution of CPB and AIW enables {\METHODNAME} to achieve a better performance on both natural and robust accuracies, as points with $\alpha > 0$ and $\beta > 0$ lie above and to the right of the baseline point ($\alpha=0$, $\beta=0$) in \cref{fig-hyper-tradeoff}, further validating our theoretical observations.
Guided by these observations, we tune hyper-parameters as follows: (i) Select $\beta^*$ based on robust accuracy with fixed $\alpha=0.5$, yielding a better tradeoff between natural and robust accuracies. For example, $\beta^*=0.8$ will be selected according to \cref{fig-hyper-rob}, with curves shift upward and rightward as $\beta$ increases from 0.0 to 0.8 in \cref{fig-hyper-tradeoff}. (ii) Fix $\beta=\beta^*$ and then choose $\alpha^*$ according to the desired tradeoff.

\textbf{Model architecture.} We also examine {\METHODNAME} on ResNet-18 in the second line of \cref{fig-hyper} (and ResNet-50, DeiT-S \cite{pmlr-v139-touvron21a} in \cref{tab-add} in appendix). We observe that {\METHODNAME} remains effective, as there exists points locating higher and further to the right than the point with $\alpha=0$ and $\beta=0$ in \cref{fig-hyper-tradeoff-rn}. However, compared with WRN-28-10, the effectiveness of {\METHODNAME} on ResNet-18 is somewhat reduced, indicating that {\METHODNAME} is more suitable for large-scale models like traditional adversarial training methods.

\textbf{Attack. } We employ CW \cite{carlini2017towards} (executed over 20 steps with a step size of $2/255$) and AutoAttack (AA) \cite{croce2020reliable} (standard version) to evaluate {\METHODNAME} under various attacks. As shown in \cref{tab-attacks}, {\METHODNAME} consistently enhances robustness across various attack strategies.

\textbf{Imbalance ratio.} We test {\METHODNAME} with various imbalance ratios (\{10, 100\} on CIFAR10-LT and \{5, 50\} on CIFAR100-LT). \cref{tab-irs-cifar10lf,tab-irs-cifar100lf} in appendix demonstrate that {\METHODNAME} consistently improves the adversarial robustness (especially on tail class) of base methods in various IR settings.

%\section{Conclusions}
%We theoretically investigate adversarial training under long-tail distributions and identify key limitations caused by the skewed training objective and unstable adversarial distributions. Further, we prove that perturbations used for training can be generated to address both adversarial vulnerability and long-tail imbalance issues. Based on these observations, we propose {\METHODNAME}, a plug-and-play framework that incorporates Class-wise Perturbation budgets for Balance (CPB) and Adversarial Intensity Warmup (AIW), to gain adaptive perturbations. {\METHODNAME} effectively enhances both standard generalization and adversarial robustness particularly for tail classes, as validated by extensive experiments.

\section{Conclusions}
We challenge the common assumption in existing long-tailed adversarial training that adversarial data inherit the same class imbalance as clean data. 
%We theoretically analyze long-tailed adversarial training and identify its limitations. Furthermore, we show that properly controlled perturbations can simultaneously mitigate adversarial vulnerability and long-tail imbalance.
We theoretically identify the limitations of long-tailed adversarial training, and show that properly controlled perturbations can simultaneously mitigate adversarial vulnerability and data imbalance.
Motivated by the insights, we propose {\METHODNAME}, a plug-and-play framework that adaptively adjusts perturbations by class-wise and iteration-wise control. Extensive experiments demonstrate that {\METHODNAME} consistently improves adversarial robustness, particularly for tail classes.

\section*{Acknowledgements}
This work was supported by the Science and Technology Major Project of Sichuan Province (2024ZDZX0003), the National Key R\&D Program of China (2024YFB3312503), the Natural Science Foundation of Sichuan Province (2024NSFTD0048), and the State Key Laboratory of Advanced Nuclear Energy Technology in Nuclear Power Institute of China (STSW-0224-0202-08-01). We also acknowledge the support of Sichuan Province Engineering Technology Research Center of Broadband Electronics Intelligent Manufacturing.

{
    \small
    \bibliographystyle{ieeenat_fullname}
    \bibliography{main}
}

% WARNING: do not forget to delete the supplementary pages from your submission 
 \clearpage
\appendix

\setcounter{page}{1}
\def\maketitleappendix
   {
   \newpage
       \onecolumn{
        \centering
        \Large
        \textbf{\thetitle}\\
        \vspace{0.5em}Appendix \\
        \vspace{1.0em}
       }
   }
\maketitleappendix

\section{Proofs of \cref{sec-diff} (Preliminaries and problem analysis)}\label{sec-proof1}

\subsection{Useful lemmas}
\begin{lemma}\label{le-diff}
Considering two arbitrary distributions $Q_1$ and $Q_2$ over instance space $\mathcal{S}$, for any hypothesis $h \in \mathcal{H}$ and loss function $\loss(\cdot, \cdot)$, 
the following bound holds:
\begin{equation*}
\begin{aligned}
\vert \naturalrisk(h, Q_1) - \naturalrisk(h, Q_2) \vert \le {c_1} \{ \sqrt{\rho^2 + 1} \, \mathcal{W}_{{c_2} + q} (Q_1, Q_2) \}^{\frac{{c_2} + q}{q}}
\end{aligned}
\end{equation*}
\end{lemma}

\begin{lemma}\label{le-diff-iter}
For an arbitrary iteration $t \in [T]$ in the adversarial training process and an arbitrary distributions $P$ over instance space $\mathcal{S}$, the following bound holds:
\begin{equation*}
\begin{aligned}
\robustrisk (h^{(T)}, P) - \naturalrisk (h^{(t)}, P_{\text{adv}}^{h^{(t - 1)}}) \le \sum_{t^\prime = t}^{T} {c_1} \{ \sqrt{{\rho^2 + 1}} \, \mathcal{W}_{{c_2} + q} (P_{\text{adv}}^{h^{(t^\prime)}}, P_{\text{adv}}^{h^{(t^\prime - 1)}}) \}^{\frac{{c_2} + q}{q}} .
\end{aligned}
\end{equation*}
\end{lemma}

\subsection{Proofs of lemmas}

\begin{proof}[Proof of \cref{le-diff}]
According to the \cref{de-risk}, with $q \ge 1$, 
\begin{equation*}
\begin{aligned}
& \left\vert \naturalrisk(h, Q_1) - \naturalrisk(h, Q_2) \right\vert \\
=& \left\vert \mathbb{E}_{(x_1, y_1) \sim Q_1, (x_2, y_2) \sim Q_2} [\loss (h(x_1), y_1) - \loss (h(x_2), y_2)] \right\vert \\
=& \left\vert \inf_{\gamma \in \Gamma (Q_1, Q_2)} \mathbb{E}_{(x_1, y_1), (x_2, y_2) \sim \gamma} [ \loss (h(x_1), y_1) - \loss (h(x_2), y_2) ] \right\vert \\
\le & \inf_{\gamma \in \Gamma (Q_1, Q_2)} \mathbb{E}_{(x_1, y_1), (x_2, y_2) \sim \gamma} \big[ \vert \loss (h(x_1), y_1) - \loss (h(x_2), y_2) \vert \big] \\
= & \inf_{\gamma \in \Gamma (Q_1, Q_2)} \big\{ \mathbb{E}_{(x_1, y_1), (x_2, y_2) \sim \gamma} \big[ \vert \loss (h(x_1), y_1) - \loss (h(x_2), y_2) \vert \big] \big\}^{q \cdot \frac{1}{q}} \\
\le & \inf_{\gamma \in \Gamma (Q_1, Q_2)} \big\{ \mathbb{E}_{(x_1, y_1), (x_2, y_2) \sim \gamma} \big[ \vert \loss (h(x_1), y_1) - \loss (h(x_2), y_2) \vert^q \big] \big\}^{\frac{1}{q}}, 
\end{aligned}
\end{equation*}
where the last line holds because $(\mathbb{E} [\vert \loss (h(x_1), y_1) - \loss (h(x_2), y_2) \vert])^q \le \mathbb{E} [\vert \loss (h(x_1), y_1) - \loss (h(x_2), y_2) \vert^q]$ (obtained by applying the Jensen's inequality). 

Since the loss function $\loss (\cdot, \cdot)$ is Holder continuous and the hypothesis $h$ is Lipschitz continuous, we can further get that 
\begin{equation*}
\begin{aligned}
& \left\vert \naturalrisk(h, Q_1) - \naturalrisk(h, Q_2) \right\vert \\
\le & \inf_{\gamma \in \Gamma (Q_1, Q_2)} \big\{ \mathbb{E}_{(x_1, y_1), (x_2, y_2) \sim \gamma} \big[ {c_1}^q (\vert h(x_1) - h(x_2) \vert + \vert y_1 - y_2 \vert)^{{c_2} + q} \big] \big\}^{\frac{1}{q}} \\
\le & \inf_{\gamma \in \Gamma (Q_1, Q_2)} \big\{ \mathbb{E}_{(x_1, y_1), (x_2, y_2) \sim \gamma} \big[ {c_1}^q ( \rho \Vert x_1 - x_2 \Vert_1 + \vert y_1 - y_2 \vert)^{{c_2} + q} \big] \big\}^{\frac{1}{q}} \\
= & {c_1} \inf_{\gamma \in \Gamma (Q_1, Q_2)} \big\{ \mathbb{E}_{(x_1, y_1), (x_2, y_2) \sim \gamma} \big[ \sqrt{\rho^2 + 1} ( \frac{\rho}{\sqrt{\rho^2 + 1}} \Vert x_1 - x_2 \Vert_1 + \frac{1}{\sqrt{\rho^2 + 1}} \vert y_1 - y_2 \vert) \big]^{{c_2} + q} \big\}^{\frac{1}{q}} \\
\le & {c_1} \inf_{\gamma \in \Gamma (Q_1, Q_2)} \big\{ \mathbb{E}_{(x_1, y_1), (x_2, y_2) \sim \gamma} \big[ \sqrt{\rho^2 + 1} ( \Vert x_1 - x_2 \Vert_1 + \vert y_1 - y_2 \vert) \big]^{{c_2} + q} \big\}^{\frac{1}{q}} \\
= & {c_1}  (\sqrt{\rho^2 + 1})^{\frac{{c_2} + q}{q}} \inf_{\gamma \in \Gamma (Q_1, Q_2)} \big\{ \mathbb{E}_{(x_1, y_1), (x_2, y_2) \sim \gamma} \big[ ( \Vert x_1 - x_2 \Vert_1 + \vert y_1 - y_2 \vert)^{{c_2} + q} \big] \big\}^{\frac{1}{{c_2} + q} \cdot \frac{{c_2} + q}{q}} \\
= & {c_1}  \big\{ \sqrt{\rho^2 + 1} \, \mathcal{W}_{{c_2} + q} (Q_1, Q_2) \big\}^{\frac{{c_2} + q}{q}} \\
\end{aligned}
\end{equation*}
\end{proof}

\begin{proof}[Proof of \cref{le-diff-iter}]
For any $t \in[T]$, 
\begin{equation*}
\begin{aligned}
& \robustrisk (h^{(T)}, P) - \naturalrisk (h^{(t)}, P_{\text{adv}}^{h^{(t - 1)}}) \\
= & \robustrisk (h^{(T)}, P) - \robustrisk (h^{(t - 1)}, P) + \robustrisk (h^{(t - 1)}, P) - \naturalrisk (h^{(t)}, P_{\text{adv}}^{h^{(t - 1)}}) \\
= & \sum_{t^\prime = t}^{T} \big\{ \robustrisk (h^{(t^\prime)}, P) - \robustrisk (h^{(t^\prime - 1)}, P) \big\} + \robustrisk (h^{(t - 1)}, P) - \naturalrisk (h^{(t)}, P_{\text{adv}}^{h^{(t - 1)}}) \\
= & \sum_{t^\prime = t}^{T} \big\{ \robustrisk (h^{(t^\prime)}, P) - \naturalrisk ({h^{(t^\prime)}, P_{\text{adv}}^{h^{(t^\prime - 1)}}}) \big\} +  \sum_{t^\prime = t}^{T} \big\{ \naturalrisk ({h^{(t^\prime)}, P_{\text{adv}}^{h^{(t^\prime - 1)}}}) - \robustrisk (h^{(t^\prime - 1)}, P) \big\} \\
& - \big( \naturalrisk (h^{(t)}, P_{\text{adv}}^{h^{(t - 1)}}) - \robustrisk (h^{(t - 1)}, P) \big) \\
= & \sum_{t^\prime = t}^{T} \big\{ \robustrisk (h^{(t^\prime)}, P) - \naturalrisk ({h^{(t^\prime)}, P_{\text{adv}}^{h^{(t^\prime - 1)}}}) \big\} 
+ \sum_{t^\prime = t+1}^{T} \big\{ \naturalrisk ({h^{(t^\prime)}, P_{\text{adv}}^{h^{(t^\prime - 1)}}}) - \robustrisk (h^{(t^\prime - 1)}, P) \big\} \\
\le &  \sum_{t^\prime = t}^{T} \big\{ \robustrisk (h^{(t^\prime)}, P) - \naturalrisk ({h^{(t^\prime)}, P_{\text{adv}}^{h^{(t^\prime - 1)}}}) \big\} \\
= &  \sum_{t^\prime = t}^{T} \big\{ \naturalrisk (h^{(t^\prime)}, P_{\text{adv}}^{h^{(t^\prime)}}) - \naturalrisk ({h^{(t^\prime)}, P_{\text{adv}}^{h^{(t^\prime - 1)}}}) \big\}, 
\end{aligned}
\end{equation*}
where the last two lines are obtained by applying \cref{eq-min} and \cref{eq-risks-realtion}. 

According to \cref{le-diff}, 
\begin{equation*}
\begin{aligned}
\naturalrisk (h^{(t)}, P_{\text{adv}}^{h^{(t)}}) - \naturalrisk (h^{(t)}, P_{\text{adv}}^{h^{(t - 1)}}) 
\le c_1 \big\{ \sqrt{\rho^2 + 1} \mathcal{W}_{{c_2} + q} (P_{\text{adv}}^{h^{(t)}}, P_{\text{adv}}^{h^{(t - 1)}} ) \big\}^{\frac{{c_2} + q}{q}}, 
\end{aligned}
\end{equation*}
and thus we can get 
\begin{equation*}
\begin{aligned}
\robustrisk (h^{(T)}, P) - \naturalrisk (h^{(t)}, P_{\text{adv}}^{h^{(t - 1)}}) 
\le \sum_{t^\prime = t}^{T} {c_1} \{ \sqrt{{\rho^2 + 1}} \, \mathcal{W}_{{c_2} + q} (P_{\text{adv}}^{h^{(t^\prime)}}, P_{\text{adv}}^{h^{(t^\prime - 1)}}) \}^{\frac{{c_2} + q}{q}}.
\end{aligned}
\end{equation*}
\end{proof}

\subsection{Proofs of theorems}
\begin{proof} [Proof of \cref{th-bound}]
According to \cref{le-diff-iter}, the following inequality holds for any $t \in[T]$ and $r_{t} \ge 0$, 
\begin{equation*}
\begin{aligned}
r_{t} \robustrisk (h^{(T)}, P) \le r_{t} \big\{ \naturalrisk (h^{(t)}, P_{\text{adv}}^{h^{(t - 1)}}) + \sum_{t^\prime = t}^{T} {c_1} \{ \sqrt{{\rho^2 + 1}} \, \mathcal{W}_{{c_2} + q} (P_{\text{adv}}^{h^{(t^\prime)}}, P_{\text{adv}}^{h^{(t^\prime - 1)}}) \}^{\frac{{c_2} + q}{q}} \big\}.
\end{aligned}
\end{equation*}

Therefore, for any set $\{r_{t}\}_{t \in [T]}$ satisfying $r_{t} \ge 0$ and $\sum_{t = 1} ^{T}  r_{t} = T$, 
\begin{equation*}
\begin{aligned}
\robustrisk (h^{(T)}, P) \le & \frac{1}{T} \sum_{t = 1} ^{T} r_{t} \Big\{ \naturalrisk (h^{(t)}, P_{\text{adv}}^{h^{(t - 1)}}) + \sum_{t^\prime = t}^{T} {c_1} \{ \sqrt{{\rho^2 + 1}} \, \mathcal{W}_{{c_2} + q} (P_{\text{adv}}^{h^{(t^\prime)}}, P_{\text{adv}}^{h^{(t^\prime - 1)}}) \}^{\frac{{c_2} + q}{q}} \Big\} \\
\le & \frac{1}{T} \sum_{t = 1} ^{T} r_{t} \naturalrisk (h^{(t)}, P_{\text{adv}}^{h^{(t - 1)}}) + \frac{1}{T} \sum_{t = 1} ^{T} r_{t} \sum_{t^\prime = t}^{T} {c_1} \{ \sqrt{{\rho^2 + 1}} \, \mathcal{W}_{{c_2} + q} (P_{\text{adv}}^{h^{(t^\prime)}}, P_{\text{adv}}^{h^{(t^\prime - 1)}}) \}^{\frac{{c_2} + q}{q}} \\
=& \frac{1}{T} \sum_{t = 1} ^{T} r_{t} \naturalrisk (h^{(t)}, P_{\text{adv}}^{h^{(t - 1)}}) + \frac{1}{T} \sum_{t = 1} ^{T} R_{t} {c_1} \{ \sqrt{{\rho^2 + 1}} \, \mathcal{W}_{{c_2} + q} (P_{\text{adv}}^{h^{(t^\prime)}}, P_{\text{adv}}^{h^{(t^\prime - 1)}}) \}^{\frac{{c_2} + q}{q}}, 
\end{aligned}
\end{equation*}
where $R_{t} = \sum_{t^\prime=1}^{t} r_{t^\prime}$.
\end{proof}

\subsection{Derivation of the decomposition}
\label{proof-decompose}
\begin{equation*}
\begin{aligned}
& \robustrisk (h^{(t - 1)}, \bar{P}) - \robustrisk (h^{(t - 1)}, {P}) \\
=& \sum_{i = 1}^{\vert \mathcal{Y} \vert} \frac{1}{\vert \mathcal{Y} \vert} \robustrisk(h^{(t - 1)}, y_i) - \sum_{i = 1}^{\vert \mathcal{Y} \vert} P(y_i) \robustrisk(h^{(t - 1)}, y_i) \\
=& \sum_{i = 1}^{\vert \mathcal{Y} \vert} \big(\frac{1}{\vert \mathcal{Y} \vert} - P(y_i) \big) \robustrisk(h^{(t - 1)}, y_i) - \sum_{i = 1}^{\vert \mathcal{Y} \vert} \frac{1}{\vert \mathcal{Y} \vert} \robustrisk(h^{(t - 1)}, y_1) + \sum_{i = 1}^{\vert \mathcal{Y} \vert} P(y_i) \robustrisk(h^{(t - 1)}, y_1)\\
=& \sum_{i = 1}^{\vert \mathcal{Y} \vert} \big(\frac{1}{\vert \mathcal{Y} \vert} - P(y_i) \big) \robustrisk(h^{(t - 1)}, y_i) - \sum_{i = 1}^{\vert \mathcal{Y} \vert} \big(\frac{1}{\vert \mathcal{Y} \vert} - P(y_i) \big) \robustrisk(h^{(t - 1)}, y_1)\\
=& \sum_{i = 2}^{\vert \mathcal{Y} \vert} \big( \frac{1}{\vert \mathcal{Y} \vert} - P(y_i) \big) \big(\robustrisk(h^{(t - 1)}, y_i) -  \robustrisk(h^{(t - 1)}, y_1) \big)
\end{aligned}
\end{equation*}

\section{Proofs of \cref{sec-effect} (Theoretical insights)}\label{sec-proof2}

\subsection{Useful lemmas}
\begin{lemma}[Class-conditional risks] \label{le-con-risk}
For a hypothesis $h \in \mathcal{H}$, the class-conditional risks are $\naturalrisk (h, y) = \probability \{ \mathcal{N}(0, 1) < \naturalzscore (h, y) \}$ and $\robustrisk (h, y) = \probability \{ \mathcal{N}(0, 1) < \robustzscore (h, y) \}$, where 
\begin{equation*}
\begin{aligned}
& \naturalzscore (h, y) = \frac{1}{ \sigma } \big( - y b - \mu_1 \Vert w_{G_1} \Vert_1 - \mu_2 \Vert w_{G_2} \Vert_1 \big),\\
& \robustzscore (h, y) = \frac{1}{ \sigma } \big( - y b - (\mu_1 - \epsilon) \Vert w_{G_1} \Vert_1 - (\mu_2 - \epsilon) \Vert w_{G_2} \Vert_1  \big).
\end{aligned}
\end{equation*}
\end{lemma}

\subsection{Proofs of lemmas}

\begin{proof}[Proof of \cref{le-con-risk}]
Focusing on the conditional natural risk $\naturalrisk(h, y)$ first, we can get that 
\begin{equation*}
\begin{aligned}
\naturalrisk(h, y) = {}& \probability \big\{yh(x) < 0 \vert y\big\} \\
= {}& \probability \big\{ y \langle w, x \rangle + yb < 0 \vert y\big\} \\
= {}& \probability \big\{ \sum_{k \in G_1 \cup G_2} y w_k \mathcal{N} ( y \theta, \sigma^2) + yb < 0 \big\} \\
= {}& \probability \big\{ \sum_{k \in G_1} w_k \mathcal{N} ( \mu_1, \sigma^2) + \sum_{k \in G_2} w_k \mathcal{N} ( \mu_2, \sigma^2) + yb < 0 \big\} \\
= {}& \probability \big\{ \mathcal{N}(0, 1) < \frac{yb - \sum_{k \in G_1} w_k \mu_1 - \sum_{k \in G_2} w_k \mu_2 }{\sqrt{\sum_{k \in G_1} w_k^2 + \sum_{k \in G_2} w_k^2} \sigma } \big\}\\
= {}& \probability \big\{ \mathcal{N}(0, 1) < \frac{yb - \mu_1 \Vert w_{G_1} \Vert_1 - \mu_2 \Vert w_{G_2} \Vert_1 }{ \sigma \Vert w \Vert_2 } \big\}\\
= {}& \probability \big\{ \mathcal{N}(0, 1) < \frac{1}{ \sigma } \big( - y b - \mu_1 \Vert w_{G_1} \Vert_1 - \mu_2 \Vert w_{G_2} \Vert_1 \big) \big\}.
\end{aligned}
\end{equation*}
Therefore, $\naturalrisk (h, y) = \probability \{ \mathcal{N}(0, 1) < \naturalzscore (h, y) \}$, where 
\begin{equation*}
\begin{aligned}
\naturalzscore (h, y) 
= \frac{1}{ \sigma } \big( - y b - \mu_1 \Vert w_{G_1} \Vert_1 - \mu_2 \Vert w_{G_2} \Vert_1 \big).
\end{aligned}
\end{equation*}

Then, we analyze the conditional robust risks. Since 
\begin{equation*}
\begin{aligned}
\robustrisk (h, y) = {}& \probability \big\{ \min_{\Vert \delta \Vert \le \epsilon} y h(x + \delta) < 0 \vert y \big\} \\
= {}& \probability \big\{ \max_{\Vert \delta \Vert \le \epsilon} y \langle w, x + \delta \rangle + yb < 0 \vert y \big\} \\
= {}& \probability \big\{ \sum_{k \in G_1 \cup G_2} \max_{\Vert \delta_k \Vert \le \epsilon} y w_k ( \mathcal{N} ( y \theta_k, \sigma^2) + \delta_k) + yb < 0 \big\} \\
= {}& \probability \big\{ \sum_{k \in G_1 \cup G_2} w_k ( \mathcal{N} ( \theta_k - \epsilon, \sigma^2)) + yb < 0 \big\} \\
= {}& \probability \big\{ \mathcal{N}(0, 1) < \frac{yb - \sum_{k \in G_1} w_k (\mu_1 - \epsilon) - \sum_{k \in G_2} w_k (\mu_2 - \epsilon) }{\sqrt{ \sum_{k \in G_1} w_k^2 + \sum_{k \in G_2} w_k^2} \sigma } \big\} \\
= {}& \probability \big\{ \mathcal{N}(0, 1) < \frac{yb - (\mu_1 - \epsilon) \Vert w_{G_1} \Vert_1 - (\mu_2 - \epsilon) \Vert w_{G_2} \Vert_1 }{ \sigma \Vert w \Vert_2 } \big\}\\
= {}& \probability \big\{ \mathcal{N}(0, 1) < \frac{1}{ \sigma } \big( - y b - (\mu_1 - \epsilon) \Vert w_{G_1} \Vert_1 - (\mu_2 - \epsilon) \Vert w_{G_2} \Vert_1  \big) \big\}, 
\end{aligned}
\end{equation*}
we can come to the conclusion that $\robustrisk (h, y) = \probability \{ \mathcal{N}(0, 1) < \robustzscore (h, y) \}$, where 
\begin{equation*}
\begin{aligned}
\robustzscore (h, y) 
= \frac{1}{ \sigma } \big( - y b - (\mu_1 - \epsilon) \Vert w_{G_1} \Vert_1 - (\mu_2 - \epsilon) \Vert w_{G_2} \Vert_1  \big). 
\end{aligned}
\end{equation*}
\end{proof}

\begin{proof}[Proof of \cref{le-bias}]
According to \cref{le-con-risk}, the sign of the difference between the two class-conditional natural risks is 
\begin{equation*}
\begin{aligned}
\sign \big( \naturalrisk (h, -1) - \naturalrisk (h, +1) \big) 
= {}& \sign \big( \probability \{ \mathcal{N}(0, 1) < \naturalzscore (h, -1) \} - \probability \{ \mathcal{N}(0, 1) < \naturalzscore (h, +1) \} \big) \\
= {}& \sign \big( \naturalzscore (h, -1) - \naturalzscore (h, +1) \big) \\
= {}& \sign (b).
\end{aligned}
\end{equation*}
Meanwhile, the sign of the difference between the two class-conditional robust risks is
\begin{equation*}
\begin{aligned}
\sign \big( \robustrisk (h, -1) - \robustrisk (h, +1) \big) 
= {}& \sign \big( \probability \{ \mathcal{N}(0, 1) < \robustzscore (h, -1) \} - \probability \{ \mathcal{N}(0, 1) < \robustzscore (h, +1) \} \big) \\
= {}& \sign \big( \robustzscore (h, -1) - \robustzscore (h, +1) \big) \\
= {}& \sign (b).
\end{aligned}
\end{equation*}
Therefore, we come to the conclusion that if the bias $b \ne 0$, $f$ will have a disparity between the two class on both natural and robust risks. Meanwhile, $\sign \big( \naturalrisk (h, -1) - \naturalrisk (h, +1) \big) = \sign \big( \robustrisk (h, -1) - \robustrisk (h, +1) \big) = \sign (b)$. 
\end{proof}

\subsection{Proofs of theorems}

\begin{proof}[Proof of \cref{th-robustness}]
Under the condition that $\epsilon_{y} \in (\mu_2, \mu_1)$, we first prove that $w^{(t)}_{i} \ge w^{(t - 1)}_{i}$ for $\forall i \in G_1$ holds by contradiction. Specifically, the contradiction to be proved is that if there exists $w_{i} < w^{(t - 1)}_{i}$ for $i \in G_{1}$ where $w_{i}$ is the $i$-th dimension of the hypothesis $h$'s weight $w$, $h \ne h^{(t)}$. 

The training objective of $h^{(t)}$ is  
\begin{equation*}
\begin{aligned}
\robustrisk (h^{(t - 1)}) 
= {}& \probability \big\{ \min_{\Vert \delta \Vert \le \epsilon_{y}}  y (\langle w^{(t - 1)}, x + \delta \rangle + b^{(t - 1)}) < 0 \big\} \\
= {}& \probability \big\{ \sum_{k \in G_1 \cup G_2} \min_{\Vert \delta_k \Vert \le \epsilon_{y}} y w^{(t - 1)}_k ( \mathcal{N} ( y \theta_k, \sigma^2) + \delta_k) + y b^{(t - 1)}  < 0 \big\} \\
= {}& \probability \big\{ \sum_{k \in G_1 \cup G_2} \min_{\Vert \delta_k \Vert \le \epsilon_{y}} w^{(t - 1)}_k ( \mathcal{N} (\theta_k + y \delta_k, \sigma^2) + y b^{(t - 1)} < 0 \big\} \\
= {}& \probability \big\{ \sum_{k \in G_1 \cup G_2} w^{(t - 1)}_k \mathcal{N} ( \theta_k - \epsilon_{y}, \sigma^2) + y b^{(t - 1)} < 0 \big\} .
\end{aligned}
\end{equation*}
Then, we can find that if there exists $i \in G_1$ satisfying $w_i < w^{(t - 1)}_i$, 
\begin{equation*}
\begin{aligned}
\robustrisk(h^{(t - 1)}) = {}& \probability \big\{ \sum_{k \in G_1 \cup G_2, k \ne i} w^{(t - 1)}_k \mathcal{N} ( \theta_k - \epsilon_{y}, \sigma^2) + w^{(t - 1)}_i \mathcal{N} ( \mu_1 - \epsilon_{y}, \sigma^2) + y b^{(t - 1)} < 0 \big\} \\
< {}& \probability \big\{ \sum_{k \in G_1 \cup G_2, k \ne i} w^{(t - 1)}_k \mathcal{N} ( \theta_k - \epsilon_{y}, \sigma^2) + w_i \mathcal{N} ( \mu_1 - \epsilon_{y}, \sigma^2) + y b^{(t - 1)} < 0 \big\}, 
\end{aligned}
\end{equation*}
which means replacing $w^{(t - 1)}_i$ by $w_i$ will result in a larger robust risk. Since the optimization objective of $h^{(t)}$ is to minimize the robust risk, $h \ne h^{(t)}$. Therefore, we come to the conclusion that $h^{(t)}$ satisfies that $w^{(t)}_{i} \ge w^{(t - 1)}_{i}$ for $\forall i \in G_1$. 

Next, we prove $w^{(t)}_{j} \le w^{(t - 1)}_{j}$ for $\forall j \in G_2$ by contradiction, which is that if there exists $w_{j} > w^{(t - 1)}_{j}$ for any $j \in G_2$ in the weight $w$ of hypothesis $h$, $h \ne h_{t}$. Since the minimization objective of $h^{(t)}$ satisfies 
\begin{equation*}
\begin{aligned}
\robustrisk (h^{(t - 1)}) = {}& \probability \big\{ \sum_{k \in G_1 \cup G_2, k \ne j} w^{(t - 1)}_k \mathcal{N} ( \theta_k - \epsilon_{y}, \sigma^2) + y b^{(t - 1)} + w^{(t - 1)}_j \mathcal{N} ( \mu_2 - \epsilon_{y}, \sigma^2) < 0 \big\} \\
< {}& \probability \big\{ \sum_{k \in G_1 \cup G_2, k \ne j} w^{(t - 1)}_k \mathcal{N} ( \theta_k - \epsilon_{y}, \sigma^2) + y b^{(t - 1)} + w_j \mathcal{N} ( \mu_2 - \epsilon_{y}, \sigma^2) < 0 \big\} , 
\end{aligned}
\end{equation*}
$w^{(t - 1)}_j$ gives a lower robust risk than $w_j > w^{(t - 1)}_j$ and thus $h \ne h^{(t)}$. Therefore, we have the conclusion that $h^{(t)}$ satisfies that $w^{(t)}_{j} \le w^{(t - 1)}_{j}$ for $\forall j \in G_2$. 

Then, we prove that $\robustrisk(h^{(t)}, y) \le \robustrisk(h^{(t - 1)}, y)$ for $y \in \mathcal{Y}$. Since according to \cref{le-con-risk}
\begin{equation*}
\begin{aligned}
\robustzscore (h^{(t)}, y) - \robustzscore (h^{(t - 1)}, y)
= \frac{1}{ \sigma} \big( (\mu_1 - \epsilon) (\Vert w_{G_1}^{(t - 1)} \Vert_1 - \Vert w_{G_1}^{(t)} \Vert_1) + (\mu_2 - \epsilon) (\Vert w_{G_2}^{(t - 1)} \Vert_1 - \Vert w_{G_2}^{(t)} \Vert_1) \big) \le 0, 
\end{aligned}
\end{equation*}
where the inequality holds because $w^{(t)}_{i} \ge w^{(t - 1)}_{i}$ for $\forall i \in G_1$ and $w^{(t)}_{j} \le w^{(t - 1)}_{j}$ for $\forall j \in G_2$, we can obtain that
\begin{equation*}
\begin{aligned}
\robustrisk (h^{(t)}, y) - \robustrisk (h^{(t - 1)}, y) 
= {}& \probability \{ \mathcal{N}(0, 1) < \robustzscore (h^{(t)}, y) \} - \probability \{ \mathcal{N}(0, 1) < \robustzscore (h^{(t - 1)}, y) \} \le 0.
\end{aligned}
\end{equation*}

\end{proof}

\begin{proof}[Proof of \cref{th-imbalance}]

According to \cref{le-bias}, $b^{(t)} = 0$ if the conditional natural/robust risks are even across different classes. We first get the exact form of optimal $b^{(t)}$, and then show what settings of class-wise perturbation intensity can lead to zero bias in $h^{(t)}$ .

We begin with the minimization objective of $h^{(t)}$, which is
\begin{equation*}
\begin{aligned}
\robustrisk (h^{(t - 1)}) \propto {}& \robustrisk (h^{(t - 1)}, -1) + K \cdot \robustrisk (h^{(t - 1)}, +1) \\
= {}& \probability \big\{ \mathcal{N}(0, 1) < \robustzscore (h^{(t - 1)}, -1) \big\} + K \cdot \probability \big\{ \mathcal{N}(0, 1) < \robustzscore (h^{(t - 1)}, +1) \big\}, 
\end{aligned}
\end{equation*}
where
\begin{equation*}
\begin{aligned}
\robustzscore (h^{(t - 1)}, y) = \frac{- y b^{(t - 1)} - (\mu_1 - \epsilon_{y}) \Vert w^{(t - 1)}_{G_1} \Vert_1 - (\mu_2 - \epsilon_{y}) \Vert w^{(t - 1)}_{G_2} \Vert_1 }{ \sigma }
\end{aligned}
\end{equation*}
according to \cref{le-con-risk}. Since the optimal $b^{(t)}$ will be found by letting $\frac{\partial \robustrisk(h^{(t - 1)}) }{\partial b^{(t - 1)}} = 0$ where
\begin{equation*}
\begin{aligned}
\frac{\partial \robustrisk(h^{(t - 1)}) }{\partial b^{(t - 1)}} = {}& \frac{1}{\sqrt{2 \pi}} \exp (- \frac{1}{2} \robustzscore^2 (h^{(t - 1)}, -1) ) \cdot \frac{\partial \robustzscore (h^{(t - 1)}, -1) }{\partial b^{(t - 1)}} \\
{}& + K \cdot \frac{1}{\sqrt{2 \pi}} \exp ( - \frac{1}{2} \robustzscore^2 (h^{(t - 1)}, +1) ) \cdot \frac{\partial \robustzscore (h^{(t - 1)}, +1) }{\partial b^{(t - 1)}} \\
= {}& \frac{1}{\sqrt{2 \pi} \sigma} \Big( \exp (- \frac{1}{2} \robustzscore^2 (h^{(t - 1)}, -1) ) - \exp (\log K - \frac{1}{2} \robustzscore^2 (h^{(t - 1)}, +1) ) \Big), 
\end{aligned}
\end{equation*}
we can obtain the following equation satisfied by optimal $b^{(t)}$
\begin{equation*}
\begin{aligned}
\frac{1}{2} \robustzscore^2 (h^{(t - 1)}, +1) - \frac{1}{2} \robustzscore^2 (h^{(t - 1)}, -1) - \log K = 0.
\end{aligned} 
\end{equation*}
Letting $A = ( \mu_1 \Vert w_{G_1}^{(t - 1)} \Vert_1 + \mu_2 \Vert w_{G_2}^{(t - 1)} \Vert_1 ) / \Vert w^{(t - 1)} \Vert_1$, we have
\begin{equation*}
\begin{aligned}
& \frac{1}{2} \robustzscore^2 (h^{(t - 1)}, +1) - \frac{1}{2} \robustzscore^2 (h^{(t - 1)}, -1) - \log K \\
=& \frac{1}{2\sigma^2} \big( {(-b^{(t - 1)} + ( \epsilon_{+1} - A) \Vert w^{(t - 1)} \Vert_1)^2 - (b^{(t - 1)} + ( \epsilon_{-1} - A) \Vert w^{(t - 1)} \Vert_1)^2} \big)  - \log K \\
=& \frac{1}{2\sigma^2} \big( {(\epsilon_{+1} - A)^2 \Vert w^{(t - 1)} \Vert_1^2 - ( \epsilon_{-1} - A)^2 \Vert w^{(t - 1)} \Vert_1^2 + ( - 2\epsilon_{+1} - 2\epsilon_{-1} + 4A) \Vert w^{(t - 1)} \Vert_1  b^{(t - 1)}} \big) - \log K. 
\end{aligned}
\end{equation*}
Therefore, the exact form of optimal $b^{(t)}$ is
\begin{equation*}
\begin{aligned}
b^{(t)} = \frac{2 \sigma^2 \log K - ( \epsilon_{+1} - A)^2 \Vert w^{(t - 1)} \Vert_1^2 + ( \epsilon_{-1} - A)^2 \Vert w^{(t - 1)} \Vert_1^2 }{(- 2\epsilon_{+1} - 2\epsilon_{-1} + 4A) \Vert w^{(t - 1)} \Vert_1  }, 
\end{aligned}
\end{equation*}
which indicates that if $( \epsilon_{+1} - A)^2 - ( \epsilon_{-1} - A)^2  = \frac{2 \sigma^2 \log K}{\Vert w^{(t - 1)} \Vert_1^2} $ holds, then $b^{(t)} = 0$. 
Therefore, we can simply find that if 
\begin{equation*}
\begin{aligned}
\epsilon_{+1} & < A - \sqrt{2} \sigma \Vert w^{(t - 1)} \Vert_1^{-1} \sqrt{\log K},\\
\epsilon_{-1} & = A - \sqrt{(\epsilon_{+1} - A)^2 - 2 \sigma^2 \Vert w^{(t - 1)} \Vert_1^{-2} \log K}, 
\end{aligned}
\end{equation*} 
the optimal $b^{(t)} = 0$. 
Meanwhile, according to \cref{le-bias}, $b^{(t)} = 0$ indicates that the conditional risks are even, i.e., $\naturalrisk (h^{(t)}, -1) - \naturalrisk (h^{(t)}, +1) = 0$ and $ \robustrisk (h^{(t)}, -1) - \robustrisk (h^{(t)}, +1)= 0$. 
\end{proof}

\begin{proof}[Proof of \cref{th-both}]
Since
\begin{equation*}
\begin{aligned}
A = \frac{\mu_1 \Vert w_{G_1}^{(t - 1)} \Vert_1 + \mu_2 \Vert w_{G_2}^{(t - 1)} \Vert_1 }{\Vert w^{(t - 1)} \Vert_1} = \frac{\mu_1 \Vert w_{G_1}^{(t - 1)} \Vert_1 + \mu_2 \Vert w_{G_2}^{(t - 1)} \Vert_1 }{\Vert w_{G_1}^{(t - 1)} \Vert_1 + \Vert w_{G_2}^{(t - 1)} \Vert_1},
\end{aligned}
\end{equation*}
we can simply find that $A$ can be bounded as $A \in [\mu_2, \mu_1]$. For the perturbation intensity $\epsilon_{+1}$, the intersection
\begin{equation*}
\begin{aligned}
\mathcal{F}_{\text{rob}} (\epsilon_{+1}) \cap \mathcal{F}_{\text{bal}} (\epsilon_{+1}) = (\mu_2, \mu_1) \cap (0, A) = (\mu_2, A - \sqrt{2} \sigma \Vert w^{(t - 1)} \Vert_1^{-1} \sqrt{\log K}) \ne \emptyset .
\end{aligned}
\end{equation*} 
Next, for the perturbation intensity $\epsilon_{-1}$, since $\mathcal{F}_{\text{bal}}(\epsilon_{-1}) = \{  A - \sqrt{ (A - \epsilon_{+1})^2 - 2\sigma^2 \Vert w^{(t - 1)} \Vert_1^{- 2} \log K } \} \ne \emptyset$ satisfies  
\begin{equation*}
\begin{aligned}
A - \sqrt{ (A - \epsilon_{+1})^2 - 2\sigma^2 \Vert w^{(t - 1)} \Vert_1^{- 2} \log K } \le A \le \mu_2,
\end{aligned}
\end{equation*} 
and 
\begin{equation*}
\begin{aligned}
A - \sqrt{ (A - \epsilon_{+1})^2 - 2\sigma^2 \Vert w^{(t - 1)} \Vert_1^{- 2} \log K } \ge A - \sqrt{(A - \epsilon_{+1})^2} = \epsilon_{+1} \ge \mu_1, 
\end{aligned}
\end{equation*} 
$\mathcal{F}_{\text{bal}}(\epsilon_{-1}) $ is in the $ \mathcal{F}_{\text{rob}}(\epsilon_{-1}) = (\mu_2, \mu_1)$, and thus $\mathcal{F}_{\text{rob}} (\epsilon_{-1}) \cap \mathcal{F}_{\text{bal}} (\epsilon_{-1}) = \mathcal{F}_{\text{bal}}(\epsilon_{-1}) \ne \emptyset$. 
\end{proof}

\subsection{Derivation of the bounds}\label{proof-bound}
For the lower bound of $\epsilon_{-1} = A - \sqrt{(A - \epsilon_{+1})^2 - (\sqrt{2} \sigma \Vert w^{(t - 1)} \Vert_1^{-1} \sqrt{\log K})^2 }$, since $(\sqrt{2} \sigma \Vert w^{(t - 1)} \Vert_1^{-1} \sqrt{\log K})^2 \ge 0$, we can get that 
 \begin{equation*}
\begin{aligned}
\epsilon_{-1} \ge A - \sqrt{(\epsilon_{+1} - A)^2} = \epsilon_{+1}.
\end{aligned}
\end{equation*}
As for the upper bound, we can find that 
 \begin{equation*}
\begin{aligned}
\epsilon_{-1}
 =& A - \sqrt{(A - \epsilon_{+1} - \sqrt{2} \sigma \Vert w^{(t - 1)} \Vert_1^{-1} \sqrt{\log K}) (A - \epsilon_{+1} + \sqrt{2} \sigma \frac{\Vert w^{(t - 1)} \Vert_2}{\Vert w^{(t - 1)} \Vert_1} \sqrt{\log K}) }\\
 \le & A - \sqrt{(A - \epsilon_{+1} - \sqrt{2} \sigma \Vert w^{(t - 1)} \Vert_1^{-1} \sqrt{\log K})^2}\\
 =& \epsilon_{+1} + \sqrt{2} \sigma \Vert w^{(t - 1)} \Vert_1^{-1} \sqrt{\log K}. 
\end{aligned}
\end{equation*}
Therefore, $\epsilon_{-1} \in [\epsilon_{+1}, \epsilon_{+1} + \sqrt{2} \sigma \Vert w^{(t - 1)} \Vert_1^{-1} \sqrt{\log K}]$. 

 \section{Reproduction} \label{sec-alg}
We provide a simple PyTorch-style pseudocode in \cref{alg-rail} for better understand on {\METHODNAME}. Meanwhile, our code repository (including the split configurations for long-tail datasets) is available at \href{https://github.com/zhang-lilin/RobustLT}{https://github.com/zhang-lilin/RobustLT}.

\lstset{
  backgroundcolor=\color{white},
  basicstyle=\fontsize{8.5pt}{9pt}\ttfamily\selectfont,
  columns=fullflexible,
  breaklines=true,
  captionpos=b,
  commentstyle=\color{darkgreen},
  keywordstyle=\color{blue1},
  stringstyle=\color{codekw},
  numberstyle=\tiny\color{codegray},
}

\begin{algorithm}[H]\label{alg-rail}
\caption{PyTorch-style pseudocode of \METHODNAME}\label{alg-rail}
\begin{lstlisting}[language=python]
"""
Args:
    alpha, beta: hyper-parameters of RobustLT
    samples_per_class: shape=[num_classes], the number of samples in each class
    base_algorithm: a given base adversarial training algorithm
    eps, step_size, steps:  parameters for adversarial example generation
Returns:
    model trained by RobustLT-enhanced base_algorithm
"""

# initial the value of classwise perturbation intensity by CPB
n_max, N = samples_per_class.max(), samples_per_class.sum()
tau = alpha / ((samples_per_class / N) * (n_max / samples_per_class).log().sqrt()).sum()
classwise_eps_max = (1 - alpha) * eps + tau * (n_max / samples_per_class).log().sqrt() * eps

# train for T epochs
for t in range(T):
    
    # calculate adversarial intensity for current epoch by AIW
    intensity = min((t - 1) / (T * beta), 1)
    classwise_eps = classwise_eps_max * intensity
    
    for (x, y) in dataloader:
        
        # generate adversarial examples with RobustLT
        eps_t = classwise_eps[y]
        step_size_t = eps_t / eps * step_size
        x_adv = base_algorithm.get_adversarial_example(x, y, eps_t, step_size_t, steps)
        
        # model update
        base_algorithm.forward(model, optimizer, x, x_adv, y)

# return the trained model
return model
\end{lstlisting}
\end{algorithm}

\section{Additional experiments}\label{app-experiment}

\subsection{Detailed configurations}

\textbf{Long-tail dataset generation.} We follow the method in \cite{cao2019learning} to generate the long-tail datasets from balanced datasets (e.g., CIFAR10/100 and TinyImagenet). Specifically, given the imbalance ratio $K$ and the number of available samples per class $N$, we then set the number of instances of the $i$-th class $N_{i} = N \cdot K^{-\frac{i-1}{\vert \mathcal{Y} \vert - 1}}$, where $i \in \{1, 2, \dots, \vert \mathcal{Y} \vert\}$. The generation are randomly conducted three times for all experiments, and the average with standard deviation are reported. 

\textbf{Other settings.} Following \cite{pangbag}, all experimental results are obtained with setting the activation function to ReLU, BN-mode to eval, the optimizer to SGD with Nesterov momentum \cite{nesterov1983method}, where learning rate, weight decay, and momentum are set to 0.1, 5e-4, and 0.9, respectively. To remain consistent with the settings of the respective base adversarial training algorithms, other hyper-parameters not explicitly mentioned, such as the learning rate schedule, batch size, and parameters for adversarial example generation, are kept the same as in their original implementations. All the experiments are conducted on an NVIDIA RTX 4090 GPU.

\begin{figure}[t] 
	\centering
	\subfloat[CIFAR10-LF]{\includegraphics[width=.34\linewidth]{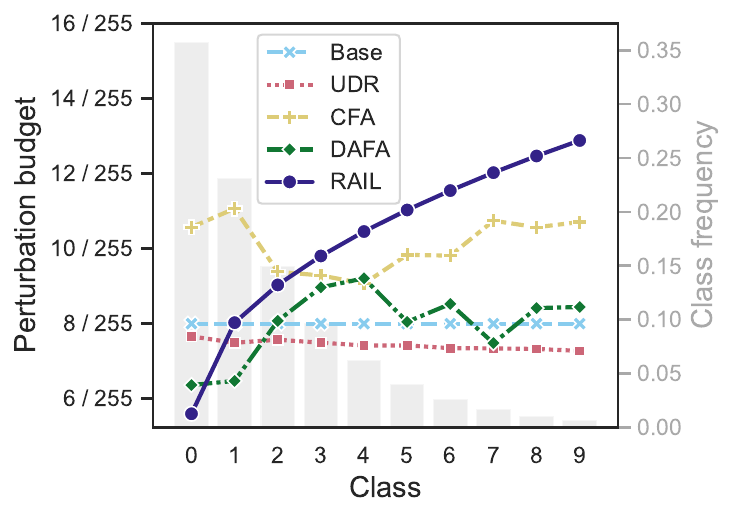} \label{fig-cpb-intensity-atbsl}} 
	\hfil
	\subfloat[CIFAR100-LF]{\includegraphics[width=.5\linewidth]{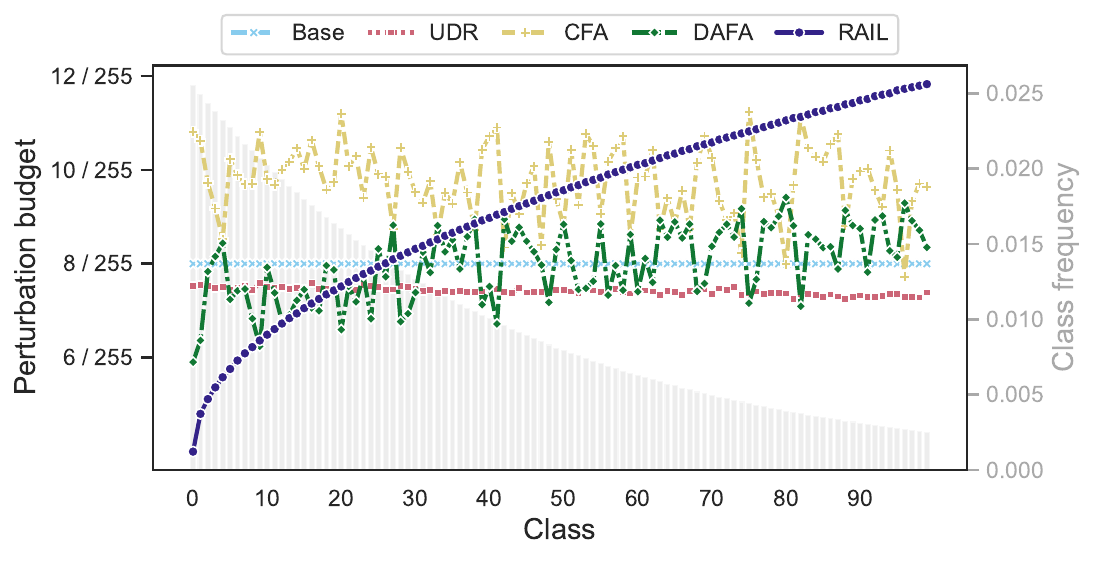} \label{fig-cpb-intensity-atbsl-cifar100}} 
	\caption{Adaptive perturbation intensity in final epoch of different enhancement methods, averaged over multiple base algorithms. }
	\label{fig-cw-atbsl}
\end{figure}

\begin{figure}[t] 
   \centering
   \subfloat[AEs against $h_1$ with CPB]{\includegraphics[width=.25\linewidth]{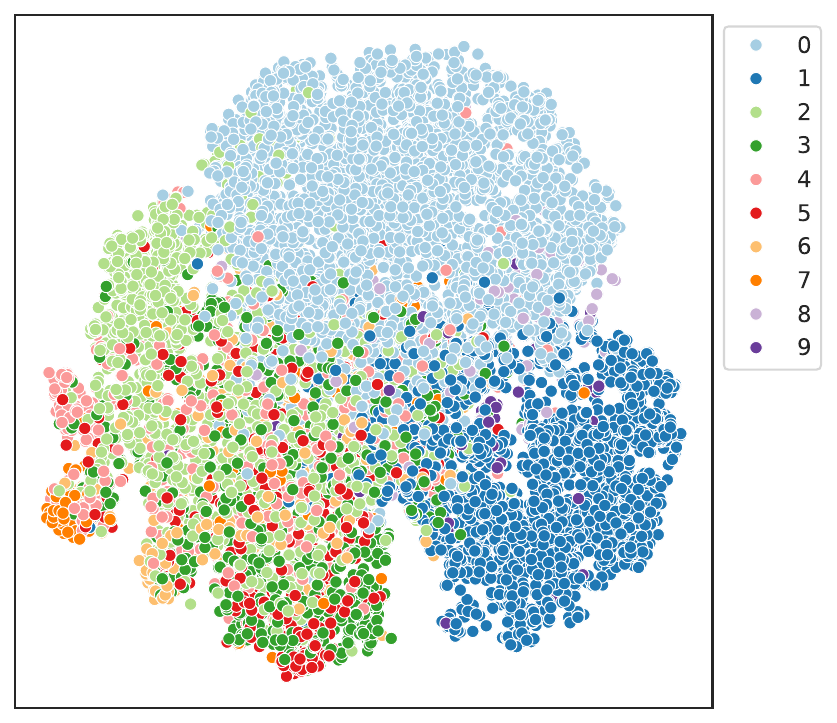} \label{fig-atbsl-rail}}
   \subfloat[AEs against $h_1$ w/o CPB]{\includegraphics[width=.25\linewidth]{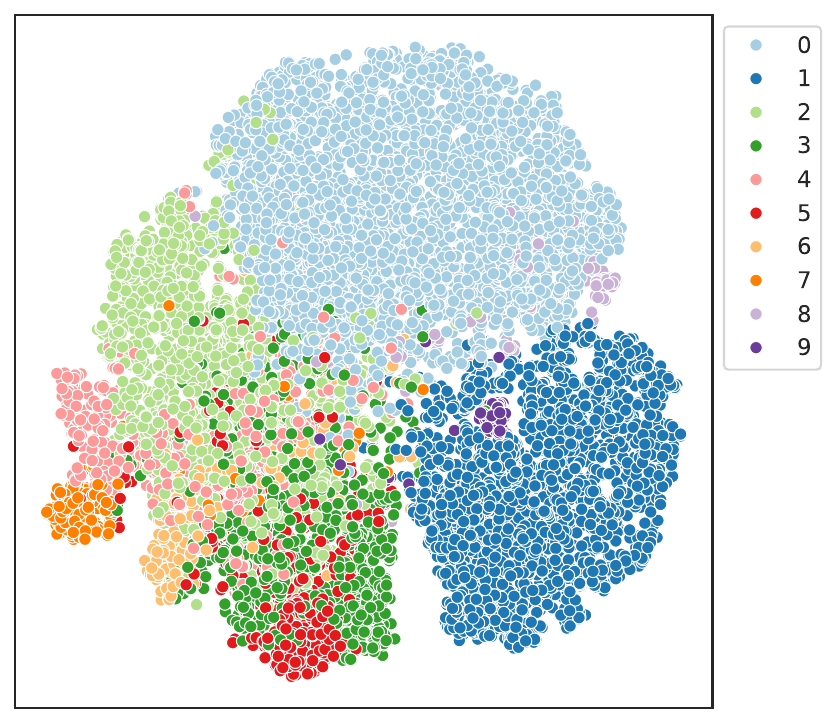} \label{fig-atbsl}} 
   \subfloat[AEs against $h_2$ with CPB]{\includegraphics[width=.25\linewidth]{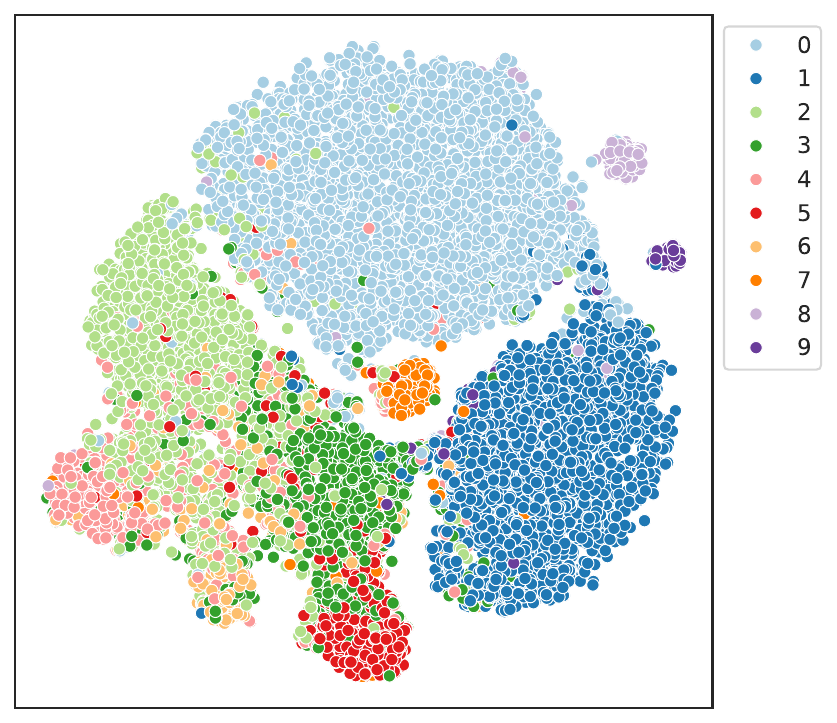} \label{fig-pgdat-rail}}
   \subfloat[AEs against $h_2$ w/o CPB]{\includegraphics[width=.25\linewidth]{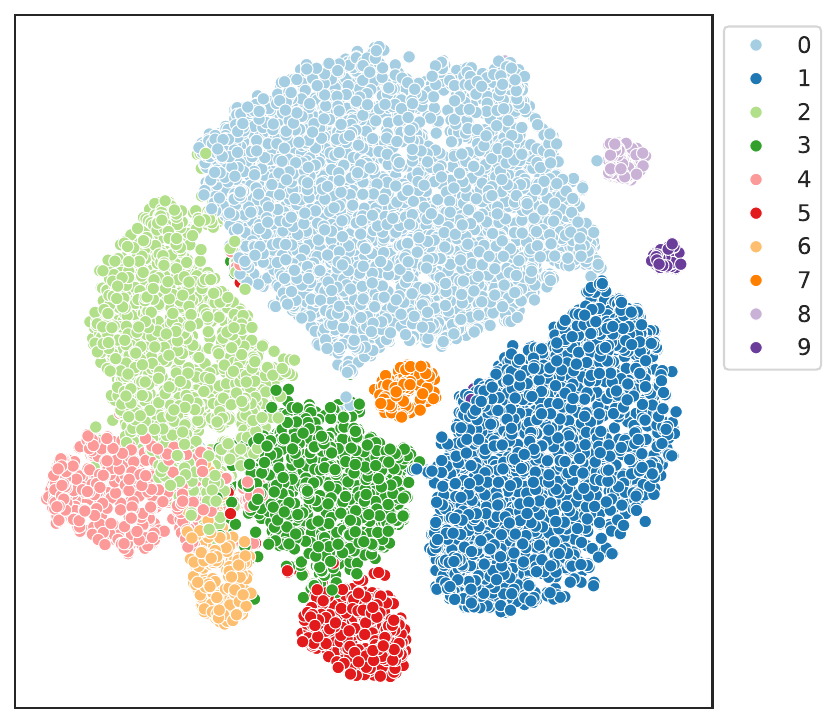} \label{fig-pgdat}}
   \caption{T-SNE visualizations of the latent space logits of adversarial examples (AEs) generated with and without CPB extracted from $h_1$ and $h_2$ on CIFAR10-LT, where $h_1$ and $h_2$ are the models trained with AT-BSL and AT, respectively. }
   \label{fig-adv-dis}
\end{figure}

\begin{figure}[t] 
   \centering
   \subfloat[Epoch 1/10 with AIW]{\includegraphics[width=.2\linewidth]{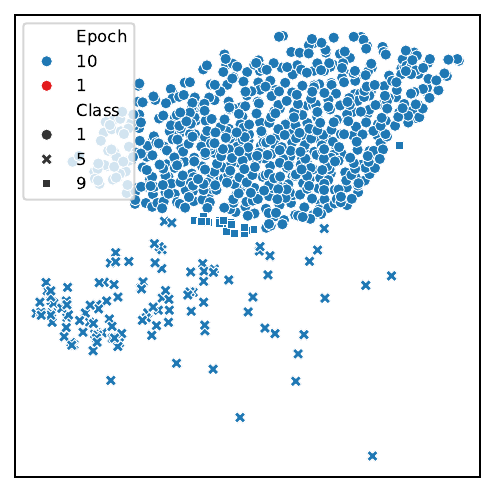}\label{fig-adv-evo-rail-1}} 
   \subfloat[Epoch 10/20 with AIW]{\includegraphics[width=.2\linewidth]{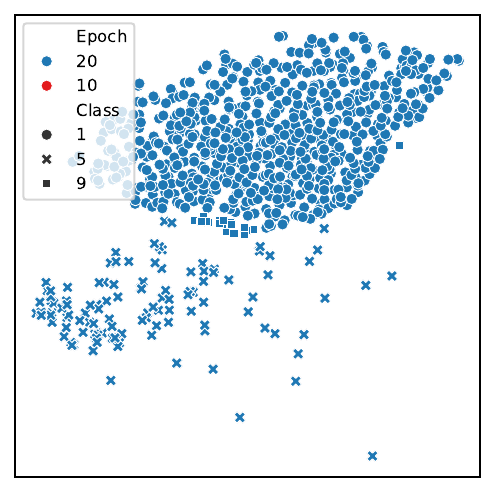}}
   \subfloat[Epoch 20/30 with AIW]{\includegraphics[width=.2\linewidth]{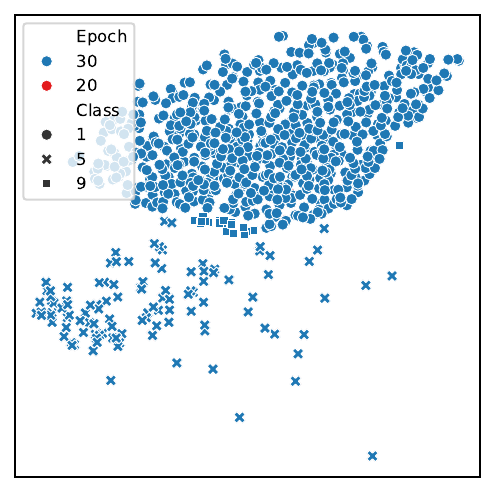}}
    \subfloat[Epoch 30/40 with AIW]{\includegraphics[width=.2\linewidth]{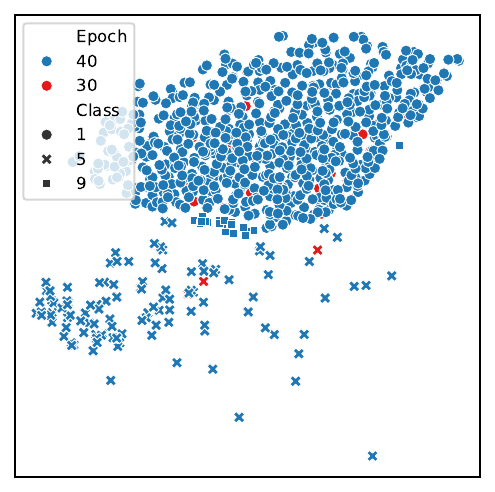}}
   \subfloat[Epoch 40/50 with AIW]{\includegraphics[width=.2\linewidth]{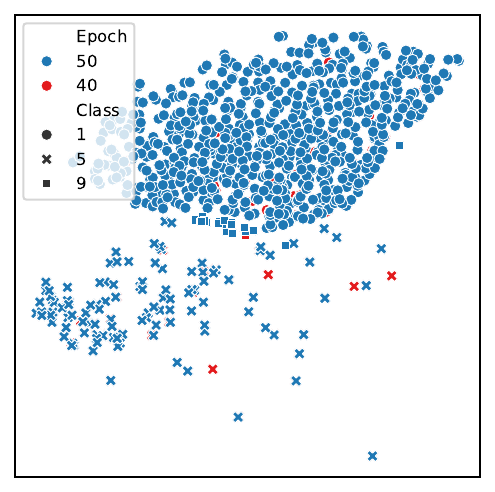}}\\
   \subfloat[Epoch 1/10 w/o AIW]{\includegraphics[width=.2\linewidth]{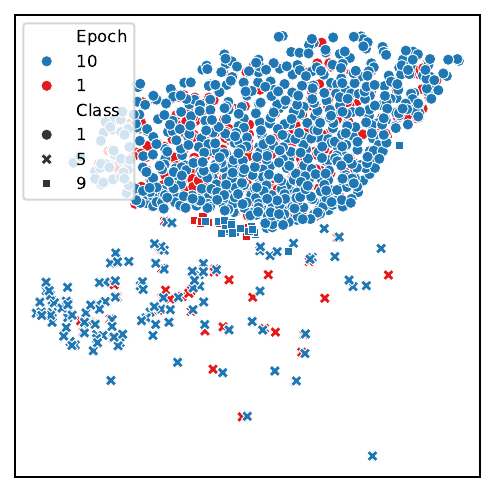}\label{fig-adv-evo-1}} 
   \subfloat[Epoch 10/20 w/o AIW]{\includegraphics[width=.2\linewidth]{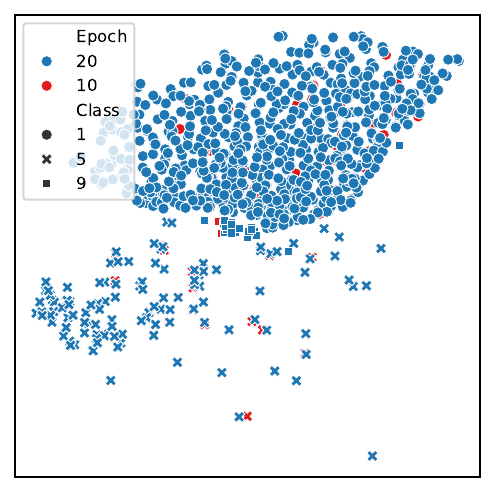}}
   \subfloat[Epoch 20/30 w/o AIW]{\includegraphics[width=.2\linewidth]{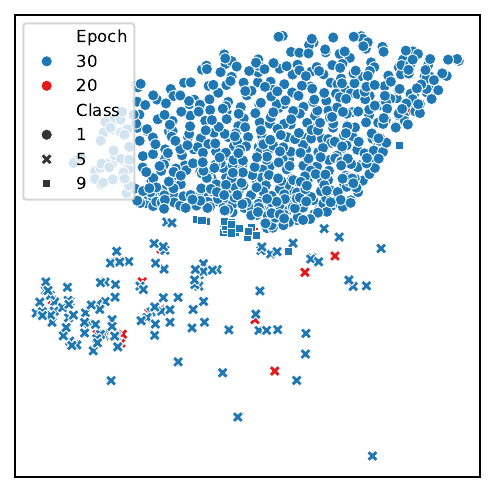}}
    \subfloat[Epoch 30/40 w/o AIW]{\includegraphics[width=.2\linewidth]{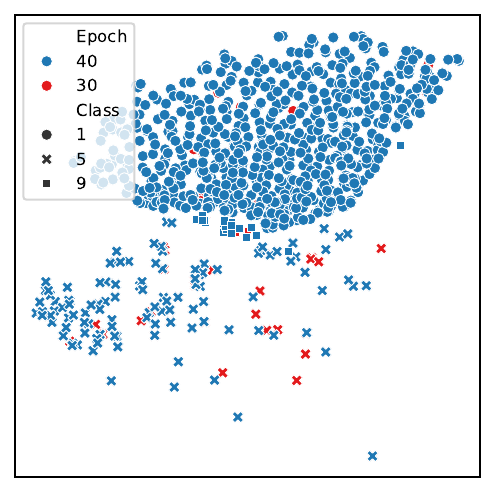}}
   \subfloat[Epoch 40/50 w/o AIW]{\includegraphics[width=.2\linewidth]{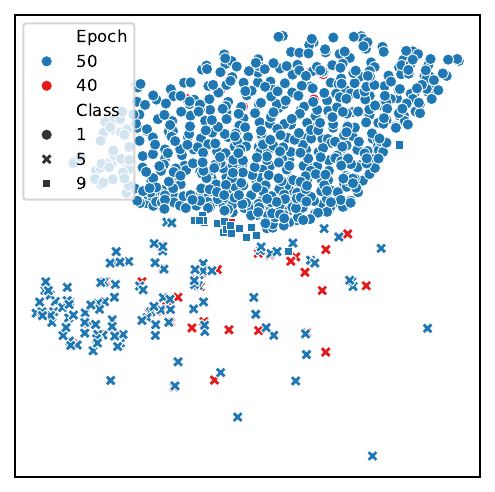}}
   \caption{T-SNE visualizations of the latent space logits of adversarial examples (AEs) generated with and without AIW across multiple epochs on CIFAR10-LT. Colors indicate training epochs, while shapes denote ground truth labels. Fewer exposed red samples reflect stronger alignment between successive adversarial distributions. }
   \label{fig-adv-evo}
\end{figure}

\subsection{Study of adversarial distribution} \label{sec-ae}
Since {\METHODNAME} consists of two components: (i) CPB, which rebalances the skewed training objective on adversarial data, and (ii) AIW, which stabilizes the evolution of adversarial distributions during training, we validate their effectiveness here.

\paragraph{Rebalanced adversarial distribution. }
Models suffering from overconfidence issue caused by data imbalance tend to generate biased adversarial examples, which in turn exacerbate the imbalance. CPB helps to rebalance the distribution of these adversarial examples. To validate this, we visualize and compare the distributions of adversarial examples generated with and without CPB in \cref{fig-adv-dis}. The results indicate that: (i) CPB encourages the generation of more diverse adversarial examples for tail-class regardless the impact of model overconfidence, as shown by the more dispersed adversarial examples of the same color in \cref{fig-atbsl-rail} and \cref{fig-pgdat-rail}, compared to \cref{fig-atbsl} and \cref{fig-pgdat}, respectively. Interpreting adversarial example generation as a form of data augmentation, these findings align with \cite{yue2024revisiting}, which highlights the importance of data augmentation in increasing sample diversity for enhancing robustness in long-tail settings. (ii) Models affected by severe data imbalance generate more biased adversarial examples, as evidenced by the narrower distribution range of tail-class adversarial examples in \cref{fig-atbsl}. This effect is more pronounced in \cref{fig-pgdat}, since models trained by AT are more sensitive to class-imbalance than that trained by AT-BSL (see \cref{tab-main}), which further validates our motivation. 

\paragraph{Stable evolution over iterations. }
To validate that AIW stabilizes the evolution of adversarial distributions, we visualize adversarial examples during training with and without AIW. Specifically, we extract checkpoints of adversarial distributions at intervals of 10 epochs and compare the alignment between successive checkpoints, as shown in \cref{fig-adv-evo}. The results show that: (i) AIW indeed stabilizes the evolution, since distributions with AIW exhibit stronger alignment, evidenced by red samples largely overlapping with blue ones in the first row of \cref{fig-adv-evo} compared to the second row. (ii) Adversarial distributions from the early training stage are less aligned than later ones, as seen in \cref{fig-adv-evo-1}, where adversarial examples from epochs 1 and 10 expose more red samples.

\begin{table}[t]
	\renewcommand{\arraystretch}{}
	\centering
	\caption{Hyper-parameter settings with respect to $(\alpha, \beta)$.}
	\label{tab-hyper-parameter}
	\setlength{\tabcolsep}{0.7cm}
	\resizebox{\textwidth}{!}{
	\begin{tabular}{lcccccc}
	\toprule
	{Dataset}&{AT}&{AWP}&{RoBal}&{REAT}&{AT-BSL}&{TAET}\\
	\midrule
	
	\multirow{1}{*}{CIFAR10-LT}&{$(0.3,0.8)$}&{$(0.3,0.8)$}&{$(0.3,0.8)$}&{$(0.3,0.8)$}&{$(0.3,0.8)$}&{$(0.3,1.0)$}\\
	\multirow{1}{*}{CIFAR100-LT}&{$(0.5,0.6)$}&{$(0.5,0.6)$}&{$(0.5,0.6)$}&{$(0.5,0.6)$}&{$(0.5,0.6)$}&{$(0.5,0.6)$}\\
	\multirow{1}{*}{TinyImagenet-LT}&{$(0.3,0.2)$}&{$(0.3,0.2)$}&{$(0.5,0.4)$}&{$(0.5,0.4)$}&{$(0.4,0.6)$}&{$(0.3,0.8)$}\\
	
	\bottomrule
	
	\end{tabular}
	}
\end{table}

\begin{table*}[t]
	\renewcommand{\arraystretch}{}
	\centering
	\caption{Natural and robust accuracies of various base adversarial training algorithms without and with {\METHODNAME} using WRN-28-10 on CIFAR10-LT under different imbalance ratios. Better results are bolded.}
	\label{tab-irs-cifar10lf}
	\setlength{\tabcolsep}{0.3cm}
	\resizebox{\textwidth}{!}{
	\begin{tabular}{llcccccccc}
	\toprule
	\multirow{2.5}{*}{Base}&\multirow{2.5}{*}{Method}&\multicolumn{4}{c}{Imbalance Ratio = 10}&\multicolumn{4}{c}{Imbalance Ratio = 100}\\
	\cmidrule(lr{0pt}){3-6}	\cmidrule(lr{0pt}){7-10}
	{}&{}&{Nat. (all)}&{Nat. (tail)}&{Rob. (all)}&{Rob. (tail)}&{Nat. (all)}&{Nat. (tail)}&{Rob. (all)}&{Rob. (tail)}\\
	\midrule
	
	{}&{origin}
	&{{74.31}${}_{\pm 0.25}$}&{{68.88}${}_{\pm 0.38}$}
	&{{35.07}${}_{\pm 0.09}$}&{{24.45}${}_{\pm 0.19}$}
	&{{51.01}${}_{\pm 0.58}$}&{{40.45}${}_{\pm 0.79}$}
	&{{25.18}${}_{\pm 0.15}$}&{{10.87}${}_{\pm 0.29}$}
	\\
	
	\rowcolor{gray!15}
	{AT}&{\METHODNAME}
	&{\best{78.07}${}_{\pm 0.27}$}&{\best{73.53}${}_{\pm 0.26}$}
	&{\best{37.15}${}_{\pm 0.15}$}&{\best{28.14}${}_{\pm 0.27}$}
	&{\best{53.12}${}_{\pm 0.02}$}&{\best{42.08}${}_{\pm 0.07}$}
	&{\best{27.23}${}_{\pm 0.05}$}&{\best{13.73}${}_{\pm 0.08}$}
	\\
	\midrule
	
	{}&{origin}
	&{{76.01}${}_{\pm 0.43}$}&{{70.78}${}_{\pm 0.55}$}
	&{{36.78}${}_{\pm 0.31}$}&{{26.08}${}_{\pm 0.28}$}
	&{{52.24}${}_{\pm 0.61}$}&{{40.88}${}_{\pm 0.83}$}
	&{{26.18}${}_{\pm 0.17}$}&{{11.74}${}_{\pm 0.30}$}
	\\

	\rowcolor{gray!15}
	{AWP}&{\METHODNAME}
	&{\best{81.04}${}_{\pm 0.26}$}&{\best{77.01}${}_{\pm 0.34}$}
	&{\best{37.22}${}_{\pm 0.24}$}&{\best{27.73}${}_{\pm 0.25}$}
	&{\best{56.67}${}_{\pm 0.75}$}&{\best{46.32}${}_{\pm 0.93}$}
	&{\best{26.49}${}_{\pm 0.21}$}&{\best{13.02}${}_{\pm 0.15}$}
	\\
	\midrule
	
	{}&{origin}
	&{{77.99}${}_{\pm 0.16}$}&{{73.33}${}_{\pm 0.11}$}
	&{{39.81}${}_{\pm 0.19}$}&{{30.26}${}_{\pm 0.21}$}
	&{\best{69.42}${}_{\pm 0.38}$}&{\best{63.16}${}_{\pm 0.36}$}
	&{{29.01}${}_{\pm 0.56}$}&{{18.45}${}_{\pm 0.58}$}
	\\
	
	\rowcolor{gray!15}
	{RoBal}&{\METHODNAME}
	&{\best{82.32}${}_{\pm 0.27}$}&{\best{79.73}${}_{\pm 0.22}$}
	&{\best{42.35}${}_{\pm 0.12}$}&{\best{36.80}${}_{\pm 0.26}$}
	&{{68.74}${}_{\pm 0.46}$}&{{63.12}${}_{\pm 0.70}$}
	&{\best{34.60}${}_{\pm 0.55}$}&{\best{26.89}${}_{\pm 0.87}$}
	\\
	\midrule
	
	{}&{origin}
	&{{77.76}${}_{\pm 0.11}$}&{{73.56}${}_{\pm 0.18}$}
	&{{36.57}${}_{\pm 0.25}$}&{{27.57}${}_{\pm 0.38}$}
	&{{64.88}${}_{\pm 0.36}$}&{{57.38}${}_{\pm 0.34}$}
	&{{26.27}${}_{\pm 0.01}$}&{{14.31}${}_{\pm 0.12}$}
	\\
	
	\rowcolor{gray!15}
	{REAT}&{\METHODNAME}
	&{\best{80.45}${}_{\pm 0.15}$}&{\best{76.87}${}_{\pm 0.22}$}
	&{\best{39.46}${}_{\pm 0.21}$}&{\best{32.17}${}_{\pm 0.16}$}
	&{\best{67.71}${}_{\pm 0.39}$}&{\best{61.35}${}_{\pm 0.44}$}
	&{\best{30.58}${}_{\pm 0.14}$}&{\best{21.17}${}_{\pm 0.10}$}
	\\
	\midrule
	
	{}&{origin}
	&{{83.47}${}_{\pm 0.21}$}&{{80.60}${}_{\pm 0.33}$}
	&{{46.42}${}_{\pm 0.19}$}&{{39.86}${}_{\pm 0.34}$}
	&{\best{75.08}${}_{\pm 0.55}$}&{\best
	{70.62}${}_{\pm 0.81}$}
	&{{34.94}${}_{\pm 0.22}$}&{{26.39}${}_{\pm 0.38}$}
	\\
	
	\rowcolor{gray!15}
	{AT-BSL}&{\METHODNAME}
	&{\best{83.63}${}_{\pm 0.07}$}&{\best{81.28}${}_{\pm 0.13}$}
	&{\best{48.86}${}_{\pm 0.32}$}&{\best{44.43}${}_{\pm 0.58}$}
	&{{73.54}${}_{\pm 0.71}$}&{{68.93}${}_{\pm 1.01}$}
	&{\best{39.25}${}_{\pm 0.11}$}&{\best{32.28}${}_{\pm 0.50}$}
	\\
	\midrule
	
	{}&{origin}
	&{{77.96}${}_{\pm 0.17}$}&{{75.37}${}_{\pm 0.21}$}
	&{{42.37}${}_{\pm 0.32}$}&{{37.55}${}_{\pm 0.22}$}
	&{{57.89}${}_{\pm 3.23}$}&{{51.92}${}_{\pm 2.91}$}
	&{{27.79}${}_{\pm 0.45}$}&{{20.28}${}_{\pm 0.44}$}
	\\
	
	\rowcolor{gray!15}
	{TAET}&{\METHODNAME}
	&{\best{79.08}${}_{\pm 0.92}$}&{\best{77.43}${}_{\pm 0.69}$}
	&{\best{42.60}${}_{\pm 0.35}$}&{\best{39.84}${}_{\pm 0.48}$}
	&{\best{62.02}${}_{\pm 1.19}$}&{\best{57.09}${}_{\pm 1.27}$}
	&{\best{30.56}${}_{\pm 0.07}$}&{\best{24.45}${}_{\pm 0.79}$}
	\\
	\bottomrule
	\end{tabular}
	}
\end{table*}

\begin{table*}[t]
	\renewcommand{\arraystretch}{}
	\centering
	\caption{Natural and robust accuracies of various base adversarial training algorithms without and with {\METHODNAME} using WRN-28-10 on CIFAR100-LT under different imbalance ratios. Better results are bolded.}
	\label{tab-irs-cifar100lf}
	\setlength{\tabcolsep}{0.3cm}
	\resizebox{\textwidth}{!}{
	\begin{tabular}{llcccccccc}
	\toprule
	\multirow{2.5}{*}{Base}&\multirow{2.5}{*}{Method}&\multicolumn{4}{c}{Imbalance Ratio = 5}&\multicolumn{4}{c}{Imbalance Ratio = 50}\\
	\cmidrule(lr{0pt}){3-6}	\cmidrule(lr{0pt}){7-10}
	{}&{}&{Nat. (all)}&{Nat. (tail)}&{Rob. (all)}&{Rob. (tail)}&{Nat. (all)}&{Nat. (tail)}&{Rob. (all)}&{Rob. (tail)}\\
	\midrule
	
	{}&{origin}
	&{{48.94}${}_{\pm 0.42}$}&{{46.00}${}_{\pm 0.44}$}
	&{{19.34}${}_{\pm 0.10}$}&{{18.03}${}_{\pm 0.09}$}
	&{{33.68}${}_{\pm 0.04}$}&{{26.53}${}_{\pm 0.05}$}
	&{{12.86}${}_{\pm 0.17}$}&{{9.37}${}_{\pm 0.30}$}
	\\
	
	\rowcolor{gray!15}
	{AT}&{\METHODNAME}
	&{\best{51.36}${}_{\pm 0.25}$}&{\best{49.63}${}_{\pm 0.32}$}
	&{\best{20.15}${}_{\pm 0.36}$}&{\best{19.65}${}_{\pm 0.43}$}
	&{\best{37.06}${}_{\pm 0.21}$}&{\best{30.62}${}_{\pm 0.17}$}
	&{\best{13.61}${}_{\pm 0.22}$}&{\best{10.75}${}_{\pm 0.32}$}
	\\
	\midrule
	
	{}&{origin}
	&{{50.55}${}_{\pm 0.11}$}&{{47.42}${}_{\pm 0.29}$}
	&{{20.87}${}_{\pm 0.05}$}&{{19.39}${}_{\pm 0.06}$}
	&{{34.59}${}_{\pm 0.07}$}&{{27.07}${}_{\pm 0.03}$}
	&{{13.65}${}_{\pm 0.27}$}&{{9.93}${}_{\pm 0.36}$}
	\\
	
	\rowcolor{gray!15}
	{AWP}&{\METHODNAME}
	&{\best{54.30}${}_{\pm 0.05}$}&{\best{52.23}${}_{\pm 0.16}$}
	&{\best{21.76}${}_{\pm 0.12}$}&{\best{20.96}${}_{\pm 0.07}$}
	&{\best{38.79}${}_{\pm 0.22}$}&{\best{31.98}${}_{\pm 0.45}$}
	&{\best{14.39}${}_{\pm 0.15}$}&{\best{11.24}${}_{\pm 0.16}$}
	\\
	\midrule
	
	{}&{origin}
	&{{53.74}${}_{\pm 0.03}$}&{{52.35}${}_{\pm 0.10}$}
	&{{21.71}${}_{\pm 0.30}$}&{{21.14}${}_{\pm 0.27}$}
	&{{40.10}${}_{\pm 0.26}$}&{{35.79}${}_{\pm 0.60}$}
	&{{14.62}${}_{\pm 0.04}$}&{{12.26}${}_{\pm 0.17}$}
	\\
	\rowcolor{gray!15}
	{RoBal}&{\METHODNAME}
	&{\best{54.04}${}_{\pm 0.12}$}&{\best{52.95}${}_{\pm 0.17}$}
	&{\best{22.17}${}_{\pm 0.24}$}&{\best{22.28}${}_{\pm 0.27}$}
	&{\best{41.48}${}_{\pm 0.43}$}&{\best{37.70}${}_{\pm 0.61}$}
	&{\best{15.61}${}_{\pm 0.07}$}&{\best{14.21}${}_{\pm 0.18}$}
	\\
	\midrule
	
	{}&{origin}
	&{{51.34}${}_{\pm 0.19}$}&{{50.11}${}_{\pm 0.18}$}
	&{{19.54}${}_{\pm 0.15}$}&{{18.95}${}_{\pm 0.20}$}
	&{{39.44}${}_{\pm 0.54}$}&{{35.77}${}_{\pm 0.46}$}
	&{{13.04}${}_{\pm 0.44}$}&{{11.18}${}_{\pm 0.54}$}
	\\
	
	\rowcolor{gray!15}
	{REAT}&{\METHODNAME}
	&{\best{51.76}${}_{\pm 0.31}$}&{\best{51.33}${}_{\pm 0.32}$}
	&{\best{20.29}${}_{\pm 0.08}$}&{\best{20.41}${}_{\pm 0.13}$}
	&{\best{40.76}${}_{\pm 0.15}$}&{\best{37.97}${}_{\pm 0.39}$}
	&{\best{14.29}${}_{\pm 0.48}$}&{\best{13.15}${}_{\pm 0.61}$}
	\\
	\midrule
	
	{}&{origin}
	&{{58.51}${}_{\pm 0.30}$}&{{57.08}${}_{\pm 0.25}$}
	&{{25.72}${}_{\pm 0.21}$}&{{25.34}${}_{\pm 0.28}$}
	&{{46.38}${}_{\pm 0.12}$}&{{42.36}${}_{\pm 0.20}$}
	&{{18.53}${}_{\pm 0.10}$}&{{16.45}${}_{\pm 0.09}$}
	\\
	
	\rowcolor{gray!15}
	{AT-BSL}&{\METHODNAME}
	&{\best{59.16}${}_{\pm 0.02}$}&{\best{57.92}${}_{\pm 0.07}$}
	&{\best{25.92}${}_{\pm 0.31}$}&{\best{26.07}${}_{\pm 0.37}$}
	&{\best{46.81}${}_{\pm 0.08}$}&{\best{42.72}${}_{\pm 0.06}$}
	&{\best{19.06}${}_{\pm 0.34}$}&{\best{17.54}${}_{\pm 0.31}$}
	\\
	\midrule
	
	{}&{origin}
	&{{49.73}${}_{\pm 0.64}$}&{{46.72}${}_{\pm 0.46}$}
	&{{21.08}${}_{\pm 0.29}$}&{{19.54}${}_{\pm 0.46}$}
	&{{35.56}${}_{\pm 0.32}$}&{{28.28}${}_{\pm 0.41}$}
	&{{13.96}${}_{\pm 0.18}$}&{{10.49}${}_{\pm 0.31}$}
	\\
	
	\rowcolor{gray!15}
	{TAET}&{\METHODNAME}
	&{\best{50.00}${}_{\pm 0.38}$}&{\best{47.75}${}_{\pm 0.23}$}
	&{\best{21.77}${}_{\pm 0.17}$}&{\best{21.16}${}_{\pm 0.25}$}
	&{\best{35.58}${}_{\pm 0.09}$}&{\best{29.11}${}_{\pm 0.18}$}
	&{\best{14.79}${}_{\pm 0.23}$}&{\best{11.77}${}_{\pm 0.26}$}
	\\
	\bottomrule
	\end{tabular}
	}
\end{table*}

\begin{table}[t]
	\renewcommand{\arraystretch}{}
	\centering
	\caption{Natural and robust accuracies of AT-BSL without and with RAIL across various datasets and model architectures. ImageNet-LT uses the first 20 classes of ImageNet64  \cite{chrabaszcz2017downsampled} with imbalance ratio 50. Adversarial training on DeiT-S follows \cite{mo2022adversarial} to use gradient clipping and pretrained initialization. \best{Better} results are highlighted. }
	\label{tab-add}
	\setlength{\tabcolsep}{0.7cm}
	\resizebox{\textwidth}{!}{
	\begin{tabular}{llccccc}
	\toprule
	\multirow{1}{*}{Dataset}&\multirow{1}{*}{Architecture}&\multirow{1}{*}{Method}&{Nat. (all)}&{Nat. (tail)}&{Rob. (all)}&{Rob. (tail)}\\
	\midrule
	
	\multirow{8}{*}{CIFAR10-LT}&\multirow{2}{*}{DeiT-S}&{origin}
	&{{58.30}${}_{\pm 1.33}$}	&{{52.74}${}_{\pm 1.62}$}
	&{{33.46}${}_{\pm 0.68}$}	&{{27.95}${}_{\pm 0.86}$}
	\\	
	{}&{}&{\METHODNAME}
	&{\best{60.83}${}_{\pm 0.51}$}	&{\best{55.34}${}_{\pm 0.76}$}
	&{\best{34.21}${}_{\pm 0.16}$}	&{\best{28.66}${}_{\pm 0.31}$}
	\\ \cmidrule{2-7}
	
	{}&\multirow{2}{*}{ResNet-18}&{origin}
	&{{72.86}${}_{\pm 0.59}$}&{{68.49}${}_{\pm 0.59}$}
	&{{38.38}${}_{\pm 0.14}$}&{{31.50}${}_{\pm 0.28}$}
	\\	
	{}&{}&{\METHODNAME}
	&{\best{73.39}${}_{\pm 0.48}$}	&{\best{69.08}${}_{\pm 0.23}$}
	&{\best{38.86}${}_{\pm 0.20}$}	&{\best{32.49}${}_{\pm 0.59}$}
	\\ \cmidrule{2-7}
	
	{}&\multirow{2}{*}{ResNet-50}&{origin}
	&{{72.86}${}_{\pm 0.88}$}&{{68.88}${}_{\pm 1.06}$}
	&{{39.01}${}_{\pm 0.16}$}&{{33.01}${}_{\pm 0.17}$}
	\\	
	{}&{}&{\METHODNAME}
	&{\best{74.29}${}_{\pm 1.40}$}	&{\best{70.41}${}_{\pm 1.23}$}
	&{\best{39.65}${}_{\pm 0.24}$}	&{\best{34.07}${}_{\pm 0.41}$}
	\\ \cmidrule{2-7}
	
	{}&\multirow{2}{*}{WRN-28-10}&{origin}
	&{{77.09}${}_{\pm 0.41}$}	&{{72.48}${}_{\pm 0.30}$}
	&{{37.98}${}_{\pm 0.42}$}	&{{28.60}${}_{\pm 0.85}$}
	\\	
	{}&{}&{\METHODNAME} 
	&{\best{77.61}${}_{\pm 0.54}$}	&{\best{73.83}${}_{\pm 0.60}$}
	&{\best{42.11}${}_{\pm 0.36}$}	&{\best{35.98}${}_{\pm 0.79}$}
	\\	\midrule

	\multirow{4}{*}{ImageNet-LT}&\multirow{2}{*}{DeiT-S}&{origin}
	&{{25.73}${}_{\pm 1.16}$}	&{{25.67}${}_{\pm 0.16}$}
	&{{16.30}${}_{\pm 0.57}$}	&{{17.25}${}_{\pm 0.97}$}
	\\	
	{}&{}&{\METHODNAME} 
	&{\best{32.73}${}_{\pm 0.37}$}	&{\best{28.46}${}_{\pm 0.21}$}
	&{\best{18.73}${}_{\pm 0.25}$}	&{\best{17.50}${}_{\pm 0.20}$}
	\\ \cmidrule{2-7}
	
	{}&\multirow{2}{*}{WRN-28-10}&{origin}
	&{{44.03}${}_{\pm 1.17}$}	&{{40.17}${}_{\pm 0.71}$}
	&{{22.63}${}_{\pm 0.12}$}	&{{22.67}${}_{\pm 0.36}$}
	\\	
	{}&{}&{\METHODNAME}
	&{\best{45.03}${}_{\pm 0.74}$}	&{\best{41.50}${}_{\pm 0.74}$}
	&{\best{23.43}${}_{\pm 0.60}$}	&{\best{23.58}${}_{\pm 0.60}$}
	\\ 
	\bottomrule
	\end{tabular}
	}
\end{table}

\begin{figure*}[t] 
	\centering
	\subfloat[Robust accuracy (CIFAR100-LT)]{\includegraphics[width=.3\linewidth]{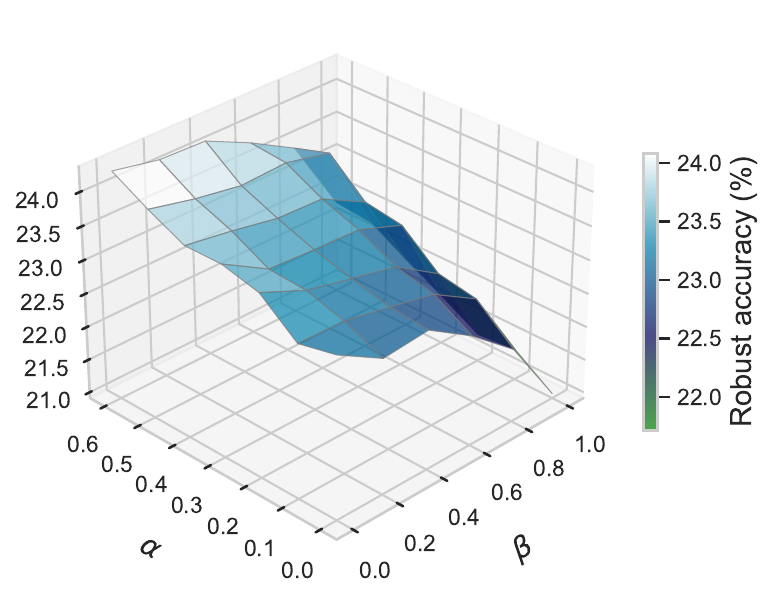} \label{fig-hyper-rob-cifar100}} \hfil
	\subfloat[Natural accuracy (CIFAR100-LT)]{\includegraphics[width=.3\linewidth]{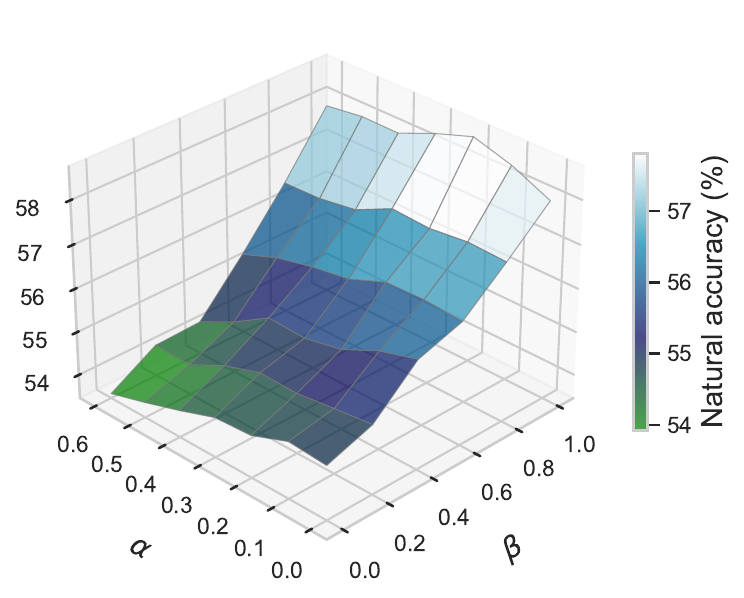} \label{fig-hyper-nat-cifar100}} 	\hfil
	\subfloat[Tradeoff (CIFAR100-LT)]{\includegraphics[width=.27\linewidth]{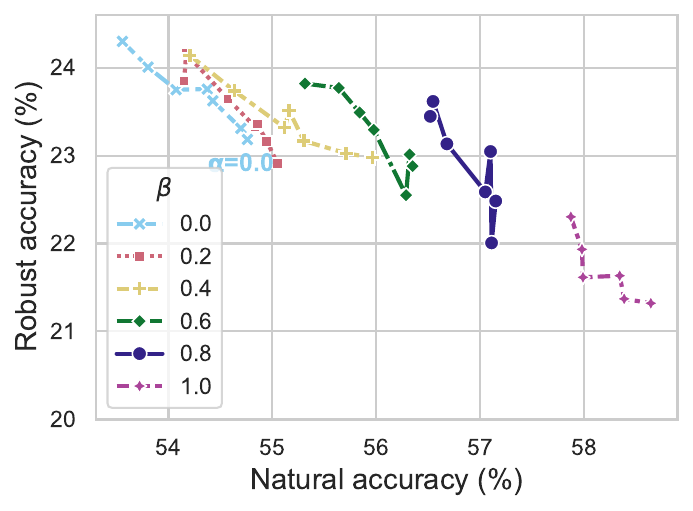} \label{fig-hyper-tradeoff-cifar100}} \\
	\subfloat[Robust accuracy (TinyImageNet-LT)]{\includegraphics[width=.3\linewidth]{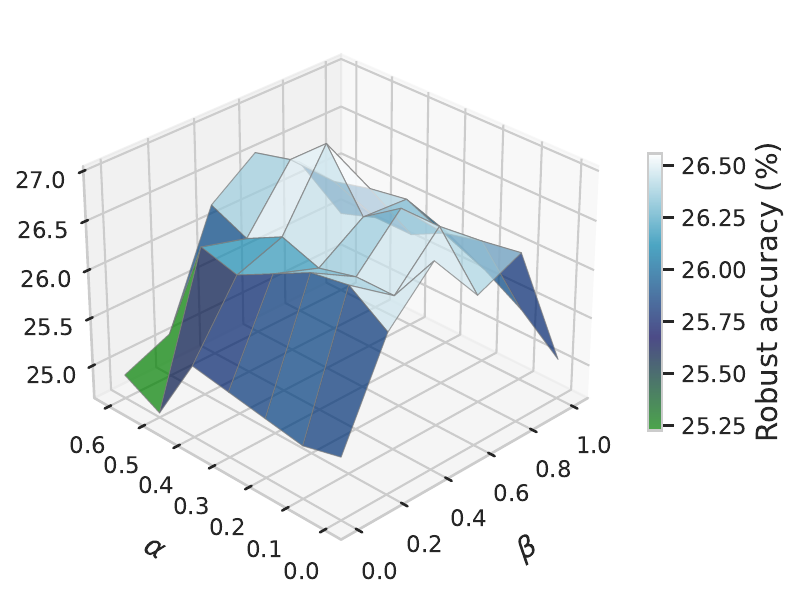} \label{fig-hyper-rob-tiny}} \hfil
	\subfloat[Natural accuracy (TinyImageNet-LT)]{\includegraphics[width=.3\linewidth]{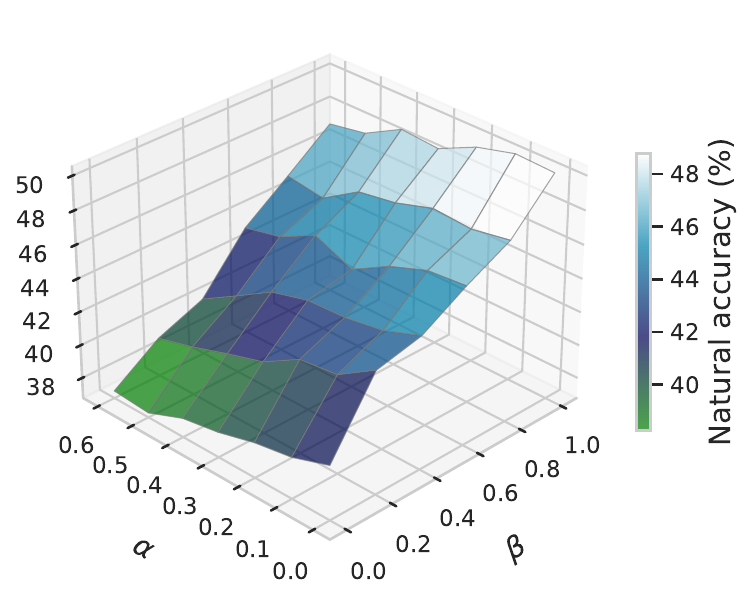} \label{fig-hyper-nat-tiny}} 	\hfil
	\subfloat[Tradeoff (TinyImageNet-LT)]{\includegraphics[width=.27\linewidth]{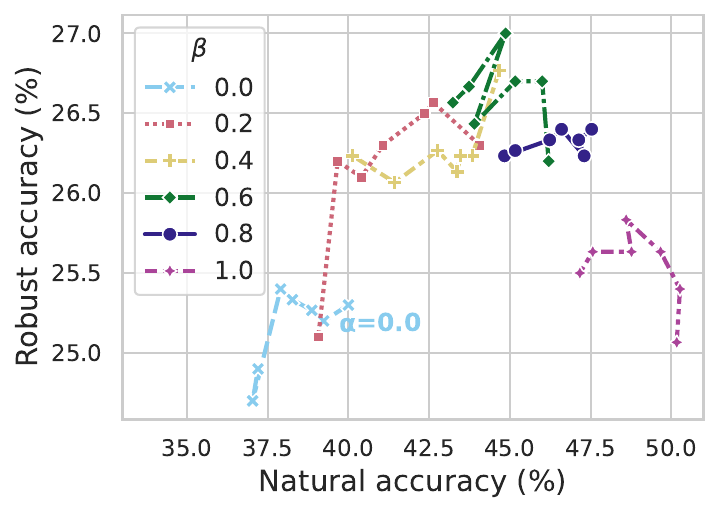} \label{fig-hyper-tradeoff-tiny}} 
	\caption{Robust accuracy, natural accuracy, and the tradeoff between them under varying settings of $\alpha$ and $\beta$ when applying to AT-BSL on CIFAR100-LT and TinyImageNet-LT. The tradeoff curves correspond to varying values of $\beta$, with individual points representing different $\alpha \in \{0.0, 0.1, 0.2, 0.3, 0.4, 0.5, 0.6\}$ in order.}
	\label{fig-hyper-tiny}
\end{figure*}

\subsection{Additional results} 

Additional experiment results are provided in \cref{tab-irs-cifar10lf,tab-irs-cifar100lf,tab-add} and \cref{fig-hyper-tiny}, the observations of which are consistent to the analysis in the main text.

\section{Discussions}\label{sec-discussion}

\subsection{Limitations}\label{sec-limitation}

\textbf{Required assumption.}
We assume that class frequencies are available in training, which may not always hold in practice. Future works going beyond this assumption are encouraged.

\paragraph{Gap between theory and methodology (from binary classification to multi-class classification).}
The conclusions in \cref{sec-effect} based on a binary task, by which we can directly define the perturbation intensity for the majority class `+1' and minority class `-1' as: $\epsilon_{+1} = (1 - \alpha) \epsilon$ and $\epsilon_{-1} = (1 - \alpha) \epsilon + \tau \sqrt{\log K} \epsilon$, where $K = \frac{P(y = +1)}{P(y = -1)}$, $\alpha$ is a hyper-parameter controlling the minimum perturbation intensity, and $\tau$ is a hyper-parameter that determines the variation in perturbation intensity between classes. 
For multi-class case where $P(y_1) \ge P(y_2) \ge \dots \ge P(y_{\vert \mathcal{Y} \vert})$, the bias corresponding to each minority class is dominated by the gap between it and the most frequent class $y_1$ according to \cref{eq-gap}. 
Therefore, we extend it to the more general multi-class case as $\epsilon_{y} = (1 -\alpha) \epsilon + \tau \sqrt{ \log K_{y}} \epsilon$, where $K_{y} = \frac{P(y_1)}{P(y)}$. This further forms CPB. 

\paragraph{Gap between theory and practice.}
(i) \textbf{About the generalization gap.} We use population risk (over distribution) to isolate the effect of class imbalance and provide qualitative insights that motivate {\METHODNAME}. In practice, there exists a generalization gap between population risk and empirical risk \cite{liu2025bridging}.
(ii) \textbf{About the tradeoff between standard generalization and adversarial robustness. }
The conclusion in \cref{le-bias} shows that the bias term is an indicator of both overconfidence phenomenons on standard generalization (reflecting by $\naturalrisk$) and adversarial robustness (reflecting by $\robustrisk$), and consequently the offsetting of negative effects caused by data imbalance takes effect simultaneously on both of the two theoretically. The reason this conclusion holds is that there is no conflict between them for the conceptual binary classification task formalized in \cref{sec-effect}. While the view that adversarial robustness is not at odd with standard generalization is supported by many works theoretically \cite{stutz2019disentangling,zhang2024provable}, there is always a tradeoff between them in practice due to the limitation of model capacity. Therefore, in our experiment, CPB, the theoretical foundation of which depends on \cref{th-both}, may not always increase both natural and robust accuracies simultaneously as shown in \cref{fig-hyper}. However, {\METHODNAME} can do this due to the incorporating of AIW as discussed in \cref{sec-sensitivity}. 
(iii) \textbf{About the non-universality of the $\sqrt{\log K}$ scaling. }
Our analysis intentionally adopts simplified settings (linear classifiers, Gaussian data) to isolate the effect of class imbalance and provide qualitative insights that motivate {\METHODNAME}. The $\sqrt{\log K}$ dependence in \cref{th-both} is not universal and relies on these assumptions. The key takeaway is the monotonic relationship between class imbalance and perturbation intensity, rather than the exact functional form.

\subsection{Warmup w.r.t. adversarial intensity in balanced scenario verses imbalanced scenario}
Warmup w.r.t. adversarial intensity is not new for adversarial training with balanced data, which often sets the warmup length to a small value (no more than 20\% of total training length on CIFAR10) and has limited effect \cite{pangbag}. Different from that, {\METHODNAME} adopts a large warmup length (e.g., 80\% of total training length on CIFAR10-LT) leading to a better performance. 

\subsection{Broader impacts}\label{sec-broader-impact}

This work addresses the intersection of adversarial robustness and class imbalance, two fundamental challenges in deploying machine learning models in real-world scenarios. We propose {\METHODNAME}, a general framework that enhances adversarial training under long-tailed distributions. 
Positive societal impacts include improved reliability and fairness in applications involving rare but critical classes (e.g., medical anomalies or uncommon traffic signs). As {\METHODNAME} is compatible with existing adversarial training methods, it lowers the barrier for broader adoption in practice. Negative societal impacts could arise if adversarial robustness is misused to build more evasive or manipulative AI systems (e.g., in misinformation or surveillance tools). We encourage transparency, auditing, and responsible use to mitigate such risks.

\end{document}